\pdfoutput=1
\documentclass[acmsmall,screen,nonacm]{acmart}

\usepackage{booktabs}
\usepackage{array}
\usepackage{multirow}
\usepackage{tabularx}
\usepackage{makecell}
\usepackage{ragged2e}
\usepackage{tikz}
\usetikzlibrary{arrows.meta,positioning,fit,backgrounds,calc,shapes.geometric}

\newcolumntype{Y}{>{\RaggedRight\arraybackslash}X}
\setcellgapes{1pt}
\makegapedcells

\AtEndPreamble{%
  \theoremstyle{acmdefinition}%
  \newtheorem{remark}[theorem]{Remark}%
}

\begin{document}

\title{Learning-Augmented Algorithms: Guarantees, Construction Mechanisms, and System-Level Implications}

\author{Hailiang Zhao}
\affiliation{%
  \institution{School of Software Technology, Zhejiang University}
  \city{Ningbo}
  \country{China}}
\email{hliangzhao@zju.edu.cn}

\author{Xueyan Tang}
\affiliation{%
  \institution{School of Computer Science and Engineering, Nanyang Technological University}
  \city{Singapore}
  \country{Singapore}}
\email{asxytang@ntu.edu.sg}

\author{Peng Chen}
\affiliation{%
  \institution{College of Computer Science and Technology, Zhejiang University}
  \city{Hangzhou}
  \country{China}}
\email{pgchen@zju.edu.cn}

\author{Zenong Ye}
\affiliation{%
  \institution{School of Software Technology, Zhejiang University}
  \city{Ningbo}
  \country{China}}
\email{yezenong@zju.edu.cn}

\author{Tianjun Zhu}
\affiliation{%
  \institution{School of Engineering, Stanford University}
  \city{Stanford}
  \country{United States}}
\email{tizhu@stanford.edu}

\author{Jianwei Yin}
\affiliation{%
  \institution{College of Computer Science and Technology, Zhejiang University}
  \city{Hangzhou}
  \country{China}}
\email{zjuyjw@zju.edu.cn}

\author{Shuiguang Deng}
\affiliation{%
  \institution{College of Computer Science and Technology, Zhejiang University}
  \city{Hangzhou}
  \country{China}}
\email{dengsg@zju.edu.cn}

\renewcommand{\shortauthors}{Zhao et al.}

\begin{abstract}
Learning-augmented algorithms use fallible predictions while retaining formal performance guarantees. This survey synthesizes prediction interfaces, error measures, consistency--robustness trade-offs, and five representative construction mechanisms across online optimization, caching, learned data structures, graph problems, and mechanism design. An orthogonal theorem-level axis distinguishes achieved upper bounds from matched asymptotic dependence. Formal guarantees are separated from empirical systems evidence, with explicit treatment of prediction cost, feedback, and composition. The resulting synthesis states sufficient conditions for limited end-to-end reasoning and delineates open problems in cost-aware prediction, endogenous error, semantic predictors, and benchmarking.
\end{abstract}


\maketitle

\section{Introduction}
\label{sec:intro}

Consider a cloud operator investigating a latency spike. The following three hypothetical scenarios illustrate failure modes that are relevant to the scope of this survey.

\begin{itemize}
    \item[\textbf{(S1)}] An autoscaler trained on earlier traffic over-provisions after the workload changes. Its prediction is stale, and the cost of updating the predictor is not included in the scaling objective.
    \item[\textbf{(S2)}] A caching layer uses a learned eviction policy tuned for average-case hit rate. Under request patterns outside its evaluation distribution, the policy has no stated prediction-independent worst-case guarantee.
    \item[\textbf{(S3)}] A query planner relies on a cardinality estimate invalidated by a schema migration and selects a poor join order. A downstream scheduler does not observe the source of the resulting delay.
\end{itemize}

These scenarios expose a recurring tension: algorithms that ignore predictable structure can be conservative, whereas learned heuristics often lack a useful worst-case statement outside their evaluation assumptions. \textit{Learning-Augmented Algorithms} (LAA) provide one formal interface between the two~\cite{mitzenmacher2021algorithms}. They admit a possibly inaccurate prediction and distinguish behavior under perfect predictions (\textit{consistency}), arbitrary predictions (\textit{robustness}), and intermediate error (often called \textit{smoothness})~\cite{purohit2018improving}. Results cover, e.g., ski rental, caching, scheduling, and matching~\cite{purohit2018improving,rohatgi2020near,lattanzi2020online,choo2025fractional}; related prediction-sensitive analyses appear in learned indexes~\cite{kraska2018case,ferragina2020pgm}. Predictor computation, feedback, retraining, and cross-component effects are usually outside the component objective. This is distinct from stationarity: robustness is commonly proved for adversarial inputs.

To connect these results with deployment, we use \textit{learning-augmented systems} (LAS) as shorthand for a system-level reading of the literature. This reading makes prediction, decision, actuation, and monitoring explicit so that omitted costs and interface assumptions can be audited. LAS is an organizing perspective in this survey, not a claim that architecture alone composes component guarantees. Section~\ref{sec:las_architecture} states limited sufficient conditions and the assumptions they require. Section~\ref{sec:llm_integration} applies the same discipline to semantic predictions. A language model may provide a typed input, but no consistency--robustness guarantee follows without a new model and proof.

\subsection{Scope}
\label{subsec:scope}

Prior surveys organize the literature by prediction type, warm-start versus online use, or advice complexity~\cite{mitzenmacher2021algorithms,mitzenmacher2022survey,boyar2017advice,boyar2023predictions}; the ALPS index maintains a living bibliography. We instead classify \emph{where and how a prediction enters a construction or analysis} and examine the assumptions needed when prediction-consuming components share compute, state, and failure domains. A direction not located in our search is labeled unverified or open in the surveyed corpus, not absent from the literature.

\subsection{Organization}
\label{subsec:organization}

Section~\ref{sec:method} states the method and limitations. Section~\ref{sec:framework} fixes the formal framework, and Section~\ref{sec:paradigms} presents five construction mechanisms plus an orthogonal evidence axis. Sections~\ref{sec:domains}--\ref{sec:open_challenges} cover domains, system composition, semantic prediction, and open problems. Recurrent qualifications concerning error measures, prediction cost, ex ante trust, and composition are made explicit where used.

\section{Scope and Method}
\label{sec:method}

This section states the questions, corpus construction, taxonomy derivation, and limitations.

\subsection{Research Questions}
\label{subsec:rq}

Four questions organize the material.

\begin{itemize}
\item[\textbf{RQ1}] What reusable technical devices convert an untrusted prediction into a provable guarantee, and under what conditions does each apply?
\item[\textbf{RQ2}] Within the surveyed corpus, which stated trade-offs or error dependencies are paired with matching lower bounds under the same model, and which are recorded only as achieved upper bounds?
\item[\textbf{RQ3}] Which assumptions of the algorithmic contract require additional modeling when the algorithm is one component of a running system?
\item[\textbf{RQ4}] What is established, what remains conditional, and what is unsupported in the recent turn toward predictions produced by foundation models?
\end{itemize}

\subsection{Corpus Construction}
\label{subsec:corpus}

\paragraph{Sources and update protocol} Entry points were prior surveys~\cite{mitzenmacher2021algorithms,mitzenmacher2022survey,boyar2017advice}, the ALPS index,\footnote{\url{https://algorithms-with-predictions.github.io/}} and early caching and ski-rental papers~\cite{lykouris2018competitive,purohit2018improving}. Citation snowballing and targeted searches of major publisher, proceedings, and preprint indexes continued through 3 September 2026. Queries combined field-level terms with application, mechanism, prediction-budget, uncertainty, and foundation-model terms. The protocol was broad but not exhaustive.

\paragraph{Review type and audit trail} This is a structured narrative survey, not a registered systematic review. Snowballing and targeted queries replaced a single database export, and no deduplicated screening log was retained. We therefore report neither PRISMA counts nor empty taxonomy cells as absence of research. The bibliography identifies the cited corpus; Table~\ref{tab:laa_taxonomy} gives selected examples of the E0/E1 coding, and Table~\ref{tab:domain_overview} maps mechanisms to named results. These records facilitate checking of representative claims without implying an exhaustive claim-level audit.

\paragraph{Inclusion and exclusion} The core corpus contains algorithms with a stated consistency, robustness, or explicit error-dependent guarantee. Learned data structures and deployed components are used only as labeled adjacent evidence about interfaces and failure modes; empirical or average-case gains are not promoted to LAA guarantees. Advice complexity is background unless coupled to an error-sensitive prediction guarantee, and offline learning of a fixed configuration lies outside the core corpus.

\paragraph{A note on verifiability} Metadata and representative claims were checked against publisher records or authoritative preprints. Interpretive assignments are author synthesis; missing combinations are questions suggested by the corpus, not absence claims.

\subsection{How the Taxonomy Was Derived}
\label{subsec:taxonomy_method}

\paragraph{Unit of classification} Axis A classifies construction mechanisms, so one paper may receive several P1--P5 labels. Axis B classifies theorem records: E0 means that this survey does not pair the upper bound with a matching lower bound in the same formal model; E1 means matched asymptotic dependence. N/A denotes empirical or adjacent records outside Axis B, not a third evidence level.

\paragraph{The two classifying questions} Axis A asks where prediction enters: P1 combines prediction-following and prediction-independent behavior; P2 modifies an optimization update while preserving feasibility; P3 couples a prediction-sensitive fractional construction with rounding; P4 controls access, switching, queries, or portfolios; and P5 uses a predicted distribution directly. Axis B asks whether the dependence is achieved (E0) or matched in the same model (E1). The mechanisms are non-exclusive and incomplete.

\paragraph{Coding procedure} For each E0/E1 row, we compared the prediction object, adversary, benchmark, objective, error, and bound assumptions at theorem level. E1 requires agreement on every model-defining dimension. No independent second coder or inter-rater statistic is reported; the tables are navigation aids, not prevalence estimates.

\subsection{Relation to Existing Classifications}
\label{subsec:paradigm_relation}

Prior organizations cut along different axes. Mitzenmacher and Vassilvitskii~\cite{mitzenmacher2021algorithms,mitzenmacher2022survey} emphasize what is predicted and the warm-start-versus-online distinction. Boyar et al.~\cite{boyar2017advice} classify by advice information, measured in bits. These views are complementary but not interchangeable: a full sequence of next-request predictions may contain far more information than an advice bound with the same per-symbol alphabet. Succinct-prediction results make that budget explicit~\cite{antoniadis2023succinct}. Table~\ref{tab:prior_surveys} states the incremental scope claimed here.

\begin{table}[t!]
\centering
\setlength{\tabcolsep}{3pt}
\caption{Relation to the principal surveys used as entry points. ``Present'' describes the scope of this article, not a claim of exhaustive coverage.}
\label{tab:prior_surveys}
\footnotesize
\begin{tabular}{@{}>{\raggedright\arraybackslash}p{2.55cm}>{\raggedright\arraybackslash}p{2.45cm}>{\raggedright\arraybackslash}p{3.37cm}>{\raggedright\arraybackslash}p{4.65cm}@{}}
\toprule
\textbf{Source} & \textbf{Primary organizing axis} & \textbf{Coverage emphasis} & \textbf{Difference from the present survey} \\
\midrule
Boyar et al.~\cite{boyar2017advice} & Advice complexity in bits & Online algorithms with advice & Predictions are treated as fallible task objects with error-dependent guarantees, not only information budgets. \\
Mitzenmacher and Vassilvitskii~\cite{mitzenmacher2021algorithms,mitzenmacher2022survey} & Predicted object; warm start vs. online use & Foundational algorithms-with-predictions examples & The present survey separates construction mechanism from tightness evidence and extends the comparison to systems interfaces. \\
Present survey & Mechanism axis P1--P5; tightness-status axis E0/E1 & Five algorithmic domains, adjacent systems, composition, and semantic interfaces & Adds theorem-level matching-bound status and conditions for cost, feedback, and cross-component reasoning. \\
\bottomrule
\end{tabular}
\end{table}

\subsection{The Shape of the Literature Over Time}
\label{subsec:timeline}

From 2018 to roughly 2022, work popularized the consistency--robustness contract and sharpened trade-offs for ski rental, caching, weighted paging, and metric tasks~\cite{purohit2018improving,lykouris2018competitive,wei2020optimal,rohatgi2020near,antoniadis2020metric,bansal2022weighted}. Later work broadened predictions to distributions, explicit learners, portfolios, and query budgets, and expanded strategic, dynamic, streaming, and graph applications~\cite{dinitz2024binarysearch,elias2024explicit,dinitz2022portfolios,im2022parsimonious,balkanski2024randomized,dong2025correlation,polak2026apsp}. Empirical benchmarks expose cost, feedback, and deployment questions~\cite{chledowski2021robust,wongkham2022updatable,wang2021cardinality}. This is a qualitative shift in emphasis; tightness remains relative to a particular model.

\subsection{Threats to Validity}
\label{subsec:threats}

Three limitations should temper how the reader uses what follows.
\begin{itemize}
  \item P1--P5 are a working set of construction mechanisms, not an exhaustive partition; papers may occupy multiple rows. E0/E1 records theorem-level evidence on a separate axis.
  \item The assignments in Table~\ref{tab:laa_taxonomy} reflect the works found by the targeted search and are qualitative. They are not paper counts or evidence that an unassigned combination is empty.
  \item The systems-oriented discussion includes empirical systems without competitive guarantees. Similarities to formal robustifiers, fallbacks, or restricted action spaces are architectural analogies unless a reduction is stated.

\end{itemize}

\section{Background and Formal Framework}
\label{sec:framework}

This section fixes the vocabulary used throughout the survey. We first give the notation and three commonly reported evaluation criteria (\S\ref{subsec:axioms}), then distinguish what a predictor outputs (\S\ref{subsec:predtypes}) from how its error is measured (\S\ref{subsec:errmeasures}). Not every paper proves all three criteria, and several use additive rather than purely multiplicative guarantees. We close with a worked ski-rental example.

\subsection{Notation}
\label{subsec:notation}

Table~\ref{tab:nomenclature} collects the notation. Let $\mathcal I$ be the instance space, $z(I)$ the information observable when a prediction is issued, and $\hat y=P(z(I))\in\mathcal Y$ the output; thus the algorithm receives no unrevealed part of $I$. For unknown target $y^*(I)$, define
\begin{equation}
    \eta(I,\hat y) = \ell_I \big(\hat{y}, y^*(I) \big),
    \label{eq_error}
\end{equation}
where $\ell_I$ is problem specific and may include the normalization used by the cited theorem. Write $\mathrm{cost}(\mathcal A(I,\hat y))$ for algorithmic cost and $\mathrm{OPT}(I)$ for offline optimum. The symbol $\mathbb E$ ranges over every random variable admitted by the theorem. For a fixed adversarial instance this may be only the algorithm's internal randomness; a stochastic or distributional model must additionally identify randomness in the realized input, target, or oracle. Unless stated otherwise, definitions concern minimization and ratios assume $\mathrm{OPT}(I)>0$; zero optima need an additive convention. Maximization results use $\mathrm{OPT}/\mathbb E[\mathrm{ALG}]$ or an explicit approximation factor.

\begin{table}[t!]
\centering
\setlength{\tabcolsep}{3pt}
\caption{Nomenclature.}
\label{tab:nomenclature}
\footnotesize
\begin{tabular}{@{}>{$}l<{$}@{\hspace{1.5em}}p{6.4cm}@{}}
\toprule
\multicolumn{1}{@{}l}{\textbf{Symbol}} & \textbf{Meaning} \\
\midrule
\mathcal{I},\, I           & Instance space; a single instance \\
\mathcal{Y},\, \hat{y}     & Prediction space; a predictor output \\
y^*(I)                     & True (unknown) label of instance $I$ \\
\eta(I,\hat y)             & Prediction error $\ell_I(\hat{y}, y^*(I))$, including stated normalization \\
\ell_I                     & Instance-aware error measure (see Table~\ref{tab:predtypes}) \\
\mathrm{OPT}(I)            & Optimal offline cost with full knowledge of $I$ \\
\mathrm{ALG},\, \mathcal{A} & The online (learning-augmented) algorithm \\
\mathrm{CR}                & Expected competitive ratio for cost minimization \\
\lambda \in [0,1]          & Trust parameter; smaller $\lambda$ means more trust here \\
c(\lambda)                 & Consistency: ratio attained when $\eta = 0$ \\
r(\lambda)                 & Robustness: worst-case ratio over all $\hat{y}$ \\
f_\lambda(\eta)           & Smoothness curve at trust $\lambda$; omit the subscript when absent \\
\sigma = (\sigma_1,\dots,\sigma_T) & Request sequence of length $T$ \\
k                          & Cache size (\S\ref{subsec:caching}) \\
H_k                        & $k$-th harmonic number $\sum_{i=1}^{k} 1/i$ \\
B                          & Purchase cost in ski rental (\S\ref{subsec:running}) \\
\bottomrule
\end{tabular}
\end{table}

\subsection{Evaluation Criteria}
\label{subsec:axioms}

A learning-augmented algorithm is commonly evaluated by the following quantities~\cite{mitzenmacher2021algorithms}. Throughout, $\lambda \in [0,1]$ denotes a \emph{trust parameter} chosen before the unrevealed part of the instance is seen. Parameter conventions differ across papers; in this survey and the running ski-rental example, smaller $\lambda$ denotes greater trust. The parameter is not a function of the realized error, which is unavailable online.

\begin{definition}[Consistency]
\label{def:consistency}
$\mathcal{A}$ is $c(\lambda)$-\emph{consistent} if, whenever the prediction is exact,
\begin{equation}
    \mathbb{E}\big[\mathrm{cost}(\mathcal{A}(I, \hat{y}))\big] \leq c(\lambda) \cdot \mathrm{OPT}(I)
    \qquad \text{for all } I,\hat y \text{ with } \eta(I,\hat y) = 0.
    \label{eq_cons}
\end{equation}
The expectation follows the probability space declared above. The special case $c(\lambda) = 1$ is called \emph{$1$-consistency}: the algorithm matches the offline optimum exactly under perfect predictions. Some results instead achieve $c(\lambda) = 1 + o(1)$ as a problem parameter grows, e.g., the horizon $T \to \infty$. We say so explicitly whenever the distinction matters.
\end{definition}

\begin{definition}[Robustness]
\label{def:robustness}
$\mathcal{A}$ is $r(\lambda)$-\emph{robust} if
\begin{equation}
    \frac{\mathbb{E}[\mathrm{cost}(\mathcal{A}(I, \hat{y}))]}{\mathrm{OPT}(I)} \leq r(\lambda)
    \qquad \text{for all } I \in \mathcal{I} \text{ and all } \hat{y} \in \mathcal{Y},
    \label{eq_rob}
\end{equation}
where the expectation follows the model's declared probability space and $r(\lambda)$ does \emph{not} depend on $\eta$. The bound must hold even for arbitrarily inaccurate predictions. Robustness need not be as small as the best prediction-free ratio $\alpha^\star$: a trade-off may accept $r(\lambda)>\alpha^\star$ in exchange for improved consistency. An algorithm can recover $\alpha^\star$ by explicitly ignoring the prediction, but a fixed trust setting need not do so.
\end{definition}

\begin{definition}[Smoothness]
\label{def:smoothness}
At trust setting $\lambda$, $\mathcal{A}$ has an $f_\lambda$-\emph{error-dependent guarantee} (often called smoothness) if there is a non-decreasing envelope $f_\lambda$ such that
\begin{equation}
    \frac{\mathbb{E}[\mathrm{cost}(\mathcal{A}(I, \hat{y}))]}{\mathrm{OPT}(I)} \leq f_\lambda(\eta(I,\hat y))
    \qquad \text{for all } I \in \mathcal{I},\ \hat{y} \in \mathcal{Y}.
    \label{eq_smooth}
\end{equation}
When $f_\lambda(0)=c(\lambda)$ the bound recovers the stated consistency guarantee, and when $f_\lambda$ is uniformly bounded it also implies robustness. We do not require continuity because error domains and published bounds may be discrete, piecewise, or include additive terms. If a result has the form $\mathbb{E}[\mathrm{cost}]\le f_\lambda(\eta)\mathrm{OPT}+b$, we report the additive constant $b$ rather than suppress it. For results without a trust parameter, the subscript is omitted.
\end{definition}

Many papers report a consistency--robustness pair traced by $\lambda$; a particular frontier is \emph{tight} only when an accompanying lower bound rules out every dominating pair in the same model.

The term ``smoothness'' is not fully standardized. Azar et al.'s discrete-smoothness separates predicted and unpredicted parts of a combinatorial solution~\cite{azar2023discrete}, whereas Definition~\ref{def:smoothness} uses the broader convention of any stated error-dependent envelope. Moreover, a Pareto-optimal consistency--robustness pair need not control performance at small nonzero error; one-way trading provides explicit brittleness examples and motivates user-specified smoothness profiles~\cite{elenter2024brittleness}.

\begin{remark}[Smoothness is not robustness]
\label{rem:smooth_not_robust}
The distinction in Definitions~\ref{def:robustness} and~\ref{def:smoothness} is substantive. A bound $\mathrm{CR}\le f_\lambda(\eta)$ with unbounded $f_\lambda$ provides no prediction-independent cap. Robustness requires a problem-specific construction, such as explicit interpolation with a baseline, an algorithm-combination theorem that accounts for state, or a formally analyzed robustifier.
\end{remark}

\subsection{Prediction Objects and Representations}
\label{subsec:predtypes}

Table~\ref{tab:predtypes} separates three levels that are easily conflated: the task object being predicted, an optional uncertainty annotation, and the representation used at the interface. The first level constrains meaningful constructions and errors but does not determine a paradigm. The latter two can qualify any task object.

\emph{Point predictions} estimate a scalar or vector; \emph{distributional predictions} specify a distribution whose calibration requirements depend on the theorem~\cite{dinitz2024binarysearch,angelopoulos2024distributional}; and \emph{ordinal predictions} give rankings that can suffice in scheduling~\cite{lindermayr2023speed}. \emph{Structural predictions} name combinatorial objects such as matchings or dual solutions~\cite{dinitz2021duals,lavastida2021instancerobust}. An \emph{uncertainty annotation} adds, e.g., a range and a coverage probability to a point prediction~\cite{sun2024uncertainty}; it is not necessarily a full predictive distribution. This model is also distinct from $\epsilon$-accurate advice, where each prediction is independently correct with probability at least $\epsilon$ and may otherwise be arbitrary~\cite{gupta2022epsilon}. A \emph{semantic representation} is free-form and requires an adapter plus a task-defined error before an LAA theorem applies (Section~\ref{sec:llm_integration}).

\begin{table}[t!]
\centering
\setlength{\tabcolsep}{3pt}
\caption{Task objects and orthogonal interface qualifiers, with typical error measures and mechanisms used in cited examples. Evidence status E0/E1 is assigned separately to individual theorems. Entries are illustrative rather than exclusive.}
\label{tab:predtypes}
\footnotesize
\begin{tabular}{@{}
    >{\raggedright\arraybackslash}p{1.73cm}
    >{\raggedright\arraybackslash}p{2.89cm}
    >{\raggedright\arraybackslash}p{2.81cm}
    >{\raggedright\arraybackslash}p{2.47cm}
    >{\raggedright\arraybackslash}p{2.81cm}
@{}}
\toprule
\textbf{Level and type} & \textbf{Predictor output} & \textbf{Typical error measure} & \textbf{Mechanisms in cited examples} & \textbf{Representative use} \\
\midrule
Task object: point
& A scalar or vector estimate of $y^*$
& $\ell_1$, $\ell_\infty$, $|\hat{y}-y^*|/y^*$
& P1--P4
& Ski rental horizon; job sizes; next request time~\cite{purohit2018improving,lykouris2018competitive} \\
\addlinespace[1pt]
Task object: distributional
& A distribution $\hat{\mathcal{D}}$ over $\mathcal{Y}$
& $W_p(\hat{\mathcal{D}}, \mathcal{D}^*)$, $d_{\mathrm{TV}}$, KL
& P4, P5
& Binary search over a predicted key distribution; contract scheduling~\cite{dinitz2024binarysearch,angelopoulos2024distributional} \\
\addlinespace[1pt]
Task object: ordinal
& A ranking or pairwise comparisons
& Kendall $\tau$; number of inversions
& P1, P3, P4
& Speed-oblivious scheduling; relative job order~\cite{lindermayr2023speed} \\
\addlinespace[1pt]
Task object: structural
& A combinatorial object: dual solution, matching, edge subset
& Symmetric difference; $\|\hat{\mathbf{y}} - \mathbf{y}^*\|_1$; edit distance
& P1--P4
& Warm-started matchings and flows; predicted dual variables~\cite{dinitz2021duals,lavastida2021instancerobust,davies2023flows} \\
\addlinespace[1pt]
Representation: semantic
& Unstructured text, code, or policy fragments
& \emph{No canonical measure}; requires task-defined semantics and a loss
& None directly; needs an adapter (\S\ref{sec:llm_integration})
& Log interpretation; fallback-policy synthesis \\
\bottomrule
\end{tabular}
\end{table}

\subsection{How Error Is Measured}
\label{subsec:errmeasures}

The loss $\ell_I$ is substantive, so curves are comparable only when error definitions, normalization, and benchmarks align. Absolute and relative error produce different curves; relative normalization also needs positive, meaningful denominators. Aggregation differs as well: $\ell_1$ charges all coordinates, whereas $\ell_\infty$ charges the worst, and neither reveals how caching errors are distributed over requests. Finally, observability is distinct from theorem validity. A late-revealed target supports ex post analysis but cannot directly drive an online controller; a deployment instead needs timely proxies, constraints, or reversible fallback rules (Section~\ref{sec:las_architecture}).

\subsection{A Running Example: Ski Rental}
\label{subsec:running}

Ski rental exhibits all three criteria. For unknown skiing days $T$, daily rent $1$, and one-time purchase cost $B$, $\mathrm{OPT}=\min\{T,B\}$. Without predictions, the optimal deterministic and randomized ratios are $2$ and $e/(e-1)\approx1.58$~\cite{karlin1988snoopy,karlin1994competitive}.

Now suppose a predictor supplies $\hat{T}$. Fix a trust parameter $\lambda \in (0,1)$ and run the deterministic rule of Purohit et al.~\cite{purohit2018improving}: if $\hat{T} \geq B$, buy at the start of day $\lceil \lambda B \rceil$; otherwise buy at the start of day $\lceil B/\lambda \rceil$. Writing $\eta_{\mathrm{abs}} = |\hat{T} - T|$, the bound stated for this rule is
\begin{equation}
    \mathrm{CR} \;\leq\; \min\!\left\{\, 1 + \tfrac{1}{\lambda},\ \ 1 + \lambda + \tfrac{\eta_{\mathrm{abs}}}{(1-\lambda)\,\mathrm{OPT}} \,\right\}.
    \label{eq_ski_cr}
\end{equation}
The three criteria can be read directly off~\eqref{eq_ski_cr}. Setting $\eta_{\mathrm{abs}} = 0$ gives $c(\lambda) = 1 + \lambda$. The first branch is independent of the error and gives $r(\lambda) = 1 + 1/\lambda$. For the normalized error $\eta(I,\hat T)=\eta_{\mathrm{abs}}/\mathrm{OPT}$, the second branch is the well-defined envelope $f_\lambda(\eta)=1+\lambda+\eta/(1-\lambda)$. The endpoint $\lambda=1$ is understood separately as the classical deterministic algorithm.

For $B=100$ and $\lambda=1/2$, the long-season branch buys on day $50$ and the short-season branch defers to day $200$, giving guarantees $c\leq1.5$ and $r\leq3$. As $\lambda\to0$, consistency approaches $1$ while robustness diverges; $\lambda=1$ recovers the classical ratio $2$. The displayed deterministic guarantee curve is asymptotically Pareto-optimal as $B$ grows; exact finite-$B$ lower bounds retain integer-rounding terms. The corresponding randomized trade-off is characterized separately~\cite{wei2020optimal}.

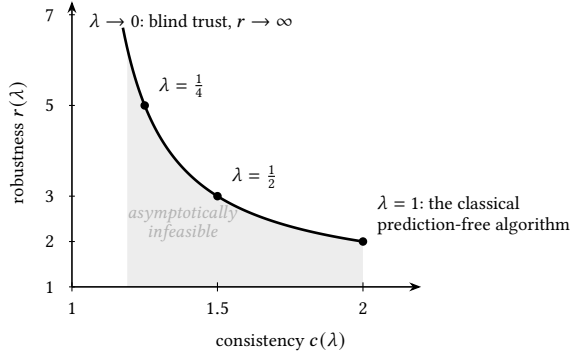
\begin{figure}[t!]
\centering
\begin{tikzpicture}[x=3.85cm, y=0.60cm, font=\footnotesize,
  ax/.style={-{Stealth[length=4.5pt]}, semithick}]

\fill[black!7] plot[domain=1.19:2.0, samples=100] (\x, {1 + 1/(\x-1)}) -- (2.0,1) -- (1.19,1) -- cycle;

\draw[ax] (1.0,1) -- (2.20,1);
\node[font=\scriptsize, anchor=north] at (1.72,0.20) {consistency $c(\lambda)$};
\draw[ax] (1.0,1) -- (1.0,7.3);
\node[rotate=90, anchor=south, font=\scriptsize] at (0.878,4.2) {robustness $r(\lambda)$};

\foreach \x/\lab in {1/$1$, 1.5/$1.5$, 2/$2$}
  {\draw (\x,0.90) -- (\x,1.10); \node[font=\scriptsize, below=1pt] at (\x,0.92) {\lab};}
\foreach \y/\lab in {1/$1$, 2/$2$, 3/$3$, 5/$5$, 7/$7$}
  {\draw (0.990,\y) -- (1.010,\y); \node[font=\scriptsize, left=2pt] at (0.990,\y) {\lab};}

\draw[line width=1pt] plot[domain=1.175:2.0, samples=120] (\x, {1 + 1/(\x-1)});

\fill (2,2) circle (1.6pt);
\node[anchor=south west, font=\scriptsize, align=left, inner sep=2pt] at (2.03,2.02)
  {$\lambda=1$: the classical\\ prediction-free algorithm};
\fill (1.5,3) circle (1.6pt);
\node[anchor=south west, font=\scriptsize, inner sep=2pt] at (1.53,3.05) {$\lambda=\tfrac12$};
\fill (1.25,5) circle (1.6pt);
\node[anchor=south west, font=\scriptsize, inner sep=2pt] at (1.28,5.05) {$\lambda=\tfrac14$};
\node[anchor=west, font=\scriptsize, align=left, inner sep=2pt] at (1.04,6.85)
  {$\lambda\to0$: blind trust, $r\to\infty$};

\node[font=\scriptsize, black!30, align=center] at (1.38,2.40) {\textit{asymptotically}\\[-1pt]\textit{infeasible}};
\end{tikzpicture}
\caption{The deterministic ski-rental guarantee curve $\big(c(\lambda),r(\lambda)\big)=\big(1+\lambda,1+1/\lambda\big)$ for $\lambda\in(0,1]$. Every point corresponds to a parameterized guarantee, and the shaded region is asymptotically unattainable as the integer purchase cost $B$ grows; finite-$B$ lower bounds include rounding terms. The endpoint $\lambda=1$ is the classical $2$-competitive rule, and exact $1$-consistency is approached only as robustness diverges.}
\label{fig:pareto}
\Description{A convex decreasing guarantee curve in the plane whose horizontal axis is consistency and whose vertical axis is robustness, showing the relation r equals one plus one over the quantity c minus one on the valid domain where consistency is greater than one and at most two. The curve rises steeply as consistency approaches one and ends at the point where consistency and robustness both equal two. Three points are marked, corresponding to trust parameter values of one, one half and one quarter. The region below and to the left of the displayed curve is shaded and labeled asymptotically infeasible.}
\end{figure}

Two features recur throughout the survey. First, the trust parameter is fixed before the realized error is known. Adaptive trust therefore requires an observable proxy and a separate theorem, such as a problem-specific algorithm-combination result. Second, the error term is additive in $\eta/\mathrm{OPT}$ but scaled by $1/(1-\lambda)$; error and trust cannot be interpreted independently. This interaction remains relevant once predictor invocations carry cost (Section~\ref{sec:open_challenges}).

\section{Design Paradigms for Learning-Augmented Algorithms}
\label{sec:paradigms}

\begin{figure}[t!]
\centering
\resizebox{\textwidth}{!}{%
\begin{tikzpicture}[
  font=\footnotesize,
  p/.style={draw, thick, rounded corners=2pt, align=center, text width=2.45cm, minimum height=1.25cm, inner sep=4pt, fill=black!6},
  e/.style={draw, rounded corners=2pt, align=center, text width=4.4cm, minimum height=0.95cm, inner sep=4pt}
]
\node[font=\small\bfseries] at (0,1.15) {Axis A: where the prediction enters the construction (non-exclusive)};
\node[p] (p1) at (-6.1,0) {\textbf{P1}\\Trust-based combination};
\node[p] (p2) at (-3.05,0) {\textbf{P2}\\Prediction-guided primal--dual update};
\node[p] (p3) at (0,0) {\textbf{P3}\\Sensitive relaxation plus rounding};
\node[p] (p4) at (3.05,0) {\textbf{P4}\\Prediction access, portfolios, and switching};
\node[p] (p5) at (6.1,0) {\textbf{P5}\\Distribution-guided decision rule};

\node[font=\scriptsize, anchor=north] at (0,-1.08) {A paper may occupy several mechanism positions; predictor type alone does not select one.};

\node[font=\small\bfseries] at (0,-2.20) {Axis B: theorem-level evidence status (orthogonal to P1--P5)};
\node[e,fill=black!2] (e0) at (-2.65,-3.25) {\textbf{E0: achieved bound}\\No matching lower bound paired in the same model};
\node[e,fill=black!10] (e1) at (2.65,-3.25) {\textbf{E1: matched dependence}\\Upper and lower bounds agree under the same formal assumptions};
\node[font=\scriptsize] at (0,-4.15) {E0/E1 records matching-bound status, not a general quality or maturity ranking.};
\end{tikzpicture}}
\caption{The two-axis taxonomy. Axis A records one or more construction mechanisms P1--P5. Axis B separately records whether a theorem provides an achieved upper bound (E0) or whether its error dependence is matched by a lower bound under the same prediction model, error measure, adversary, benchmark, and computational assumptions (E1). Neither coordinate follows from the predictor type alone.}
\label{fig:paradigm_tree}
\Description{A two-axis classification diagram. The horizontal construction axis contains five non-exclusive boxes: trust-based combination, prediction-guided primal-dual update, prediction-sensitive relaxation and rounding, prediction access and portfolios with switching, and distribution-guided decision rules. A separate tightness-status axis contains an achieved-bound state E0 and a matched-dependence state E1. The two states are not presented as a general quality ranking.}
\end{figure}

Axis A uses five labels to organize construction mechanisms that are recurring or representative in the surveyed corpus. They are non-exclusive and not equally broad; in particular, P3 is represented by a narrower line of work than P1, P2, P4, or P5. Axis B separately records tightness status at theorem level. Table~\ref{tab:paradigm_summary} summarizes both axes, and Figure~\ref{fig:paradigm_tree} shows how to read a pair of coordinates.

\begin{table}[t!]
\centering
\setlength{\tabcolsep}{3pt}
\caption{The two taxonomy axes. P1--P5 describe construction mechanisms; E0/E1 describes theorem-level evidence. A label alone implies no guarantee.}
\label{tab:paradigm_summary}
\footnotesize
\begin{tabular}{@{}
    >{\centering\arraybackslash}p{0.57cm}
    >{\raggedright\arraybackslash}p{2.69cm}
    >{\raggedright\arraybackslash}p{3.08cm}
    >{\raggedright\arraybackslash}p{1.84cm}
    >{\raggedright\arraybackslash}p{3.08cm}
    >{\raggedright\arraybackslash}p{1.23cm}
@{}}
\toprule
& \textbf{Core mechanism} & \textbf{What is guaranteed} & \textbf{Typical scope} & \textbf{Failure mode} & \textbf{Refs.} \\
\midrule
\textbf{P1}
& A single algorithm parameterized by a trust value $\lambda$ fixed offline, or a randomized mixture of two
& A theorem-specific pair $\big(c(\lambda), r(\lambda)\big)$ when both branches are controlled
& \textit{Online problems with a known prediction-free baseline}
& Needs that baseline to exist; dynamic switching additionally requires state-migration accounting
& \cite{purohit2018improving,gollapudi2019rentorbuy,wei2020optimal,grigorescu2025packing} \\
\addlinespace[3pt]
\textbf{P2}
& Prediction initializes, guides, or repairs a primal or dual certificate in an optimization method
& A problem-specific competitive, approximation, or running-time bound from feasibility and dual-based analysis
& \textit{Online covering / packing; combinatorial warm starts}
& Requires an explicit relaxation, advice semantics, and invariant; a predicted dual alone gives no generic guarantee
& \cite{bamas2020primal,bamas2020energy,dinitz2021duals,sakaue2022dca,grigorescu2025packing} \\
\addlinespace[3pt]
\textbf{P3}
& Prediction guides a fractional solution; a separate online rounding theorem makes it integral
& Product of the fractional guarantee and the rounding factor, when their assumptions align
& \textit{Online scheduling and facility location}
& Rounding may dominate the ratio and need not preserve consistency
& \cite{lattanzi2020online,lattanzi2026facility} \\
\addlinespace[3pt]
\textbf{P4}
& Control which predictor or algorithm is consulted or followed, including query, switching, and state costs
& Competition with a specified benchmark or a bound using fewer prediction calls, under a problem-specific reduction
& \textit{Multiple predictors or costly access to one predictor}
& Prediction regret alone does not imply cost competitiveness or robustness
& \cite{anand2022multiple,dinitz2022portfolios,antoniadis2023mixing,cosa2025bandit,drygala2023costly,im2022parsimonious,sadek2024reduced} \\
\addlinespace[3pt]
\textbf{P5}
& Use a predicted distribution directly in a search, schedule, or other problem-specific rule
& Error-dependent performance in a stated distributional distance, plus a separate worst-case bound when proved
& \textit{Distributional predictions}
& No generic Wasserstein or risk theorem transfers across objectives
& \cite{dinitz2024binarysearch,angelopoulos2024distributional,angelopoulos2025bidding} \\
\midrule
\multicolumn{2}{@{}l}{\textbf{Axis B: evidence status}}
& \textbf{E0}: achieved upper bound; \textbf{E1}: matching lower bound under the same formal model
& \textit{Assigned per theorem}
& Tightness does not transfer across error measures, adversaries, benchmarks, or computational models
& \cite{rohatgi2020near,antoniadis2023metric,dinitz2024binarysearch} \\
\bottomrule
\end{tabular}
\end{table}

\begin{enumerate}
    \item \emph{\textbf{P1}: Trust-Based Combination.}
    Let $\mathcal{A}_c$ be a \emph{prediction-following} algorithm with consistency factor $c_c$, potentially unbounded when the prediction is inaccurate, and let $\mathcal{A}_r$ be a prediction-free algorithm with worst-case ratio $\alpha_r$. P1 combines them under a trust parameter $\lambda \in [0,1]$ fixed before the instance is revealed. Two distinct constructions share this name, and they are analyzed by different tools.

    \emph{(i) Randomized mixture.} To keep the survey's convention that smaller $\lambda$ means more trust, draw once at the start and run $\mathcal{A}_c$ with probability $1-\lambda$ and $\mathcal{A}_r$ with probability $\lambda$. By linearity of expectation,
    \begin{equation}
        \mathbb{E}[\mathrm{cost}] = (1-\lambda) \cdot \mathbb E[\mathrm{cost}(\mathcal{A}_c)] + \lambda\cdot \mathbb E[\mathrm{cost}(\mathcal{A}_r)],
        \label{eq_p1_mix}
    \end{equation}
    which yields $c(\lambda) \leq (1-\lambda)c_c + \lambda\alpha_r$; the displayed factor reduces to $(1-\lambda)+\lambda\alpha_r$ when $\mathcal{A}_c$ is $1$-consistent. If $\mathcal{A}_c$ has unbounded worst-case ratio, every mixture with $\lambda<1$ remains non-robust; randomization alone does not cap that branch.

    \emph{(ii) Parameterized interpolation.} A different construction builds a \emph{single} algorithm whose decision rule is continuously deformed by $\lambda$, as in the ski rental rule of \S\ref{subsec:running}, where $\lambda$ moves the purchase day between $\lceil \lambda B \rceil$ and $\lceil B/\lambda \rceil$. Here $\lambda$ does not select between two codebases. It tunes one. The design goal is to trace the Pareto frontier
    \begin{equation}
        \mathcal{P} = \Big\{ \big(c(\lambda), r(\lambda)\big) : \lambda \in (0,1] \Big\},
        \label{eq_p1_pareto}
    \end{equation}
    and a result is tight when no algorithm attains a pair dominating a point of $\mathcal{P}$. For ski rental, $(1+\lambda,\,1+1/\lambda)$ is the standard deterministic guarantee curve and is asymptotically Pareto-optimal as $B$ grows; finite-$B$ statements retain integer-rounding terms. The randomized trade-off is characterized separately~\cite{purohit2018improving,gollapudi2019rentorbuy,wei2020optimal}.

    The parameter $\lambda$ is not a function of the realized error unless that error is observable before the corresponding decision. Adaptive trust needs a causal proxy and a problem-specific analysis; neither a generic cost monitor nor a generic expert algorithm supplies a competitive ratio automatically.

    \emph{Applies when} a robust prediction-free baseline exists and either an up-front mixture or a single parameterized rule yields the stated guarantee. A one-shot mixture draws one branch before execution, so it neither switches online nor requires the two branches to share state. State compatibility and migration cost become additional obligations only for dynamic switching, which is classified under P4. This distinction matters in caching, where separately maintained cache states can be expensive to reconcile.

    \item \emph{\textbf{P2}: Prediction-Guided Primal--Dual Methods.}
    For online covering problems, Bamas et al.~\cite{bamas2020primal} treat a prediction as advice specifying a candidate solution or action sequence. The algorithm modifies the rates at which primal and dual variables increase, with $\lambda$ controlling how strongly the advice influences those rates. Feasibility and dual fitting then yield a bound of the form
    \begin{equation}
       \mathrm{cost}(\mathcal A) \leq
       \min\{C(\lambda)\,S(\hat y,I),\;R(\lambda)\,\mathrm{OPT}(I)\},
    \end{equation}
    where $S(\hat y,I)$ is the cost of following the advice on the realized instance. The exact updates and constants depend on the covering formulation; there is no universal additive ``predicted-dual bias'' whose cumulative magnitude alone proves robustness.

    A distinct use of predictions is an offline warm start. Dinitz et al.~\cite{dinitz2021duals} predict dual variables for matching and measure how repairing or augmenting that prediction affects running time. Sakaue and Oki~\cite{sakaue2022dca} extend this direction through discrete convex analysis, with applications including weighted perfect bipartite matching, weighted matroid intersection, and discrete energy minimization; their bounds use distance to an optimal discrete-convex solution. These works belong under P2 because feasibility and optimization certificates drive the analysis, but they predict different objects and improve running time rather than online cost. P2 applies only after the paper identifies the optimization formulation, advice semantics, and invariant that survives inaccurate advice.

    \item \emph{\textbf{P3}: Prediction-Sensitive Relaxation and Rounding.}
    A representative example is online scheduling via learned weights~\cite{lattanzi2020online}. The prediction is used to construct a fractional assignment whose competitive ratio depends on prediction error; a separate online rounding theorem converts an eligible fractional assignment into an integral schedule. A newer line develops online rounding specifically for facility location and applies it to learning-augmented integral facility location~\cite{lattanzi2026facility}. If the fractional stage is $g(\eta)$-competitive and the rounding incurs factor $\rho$, the combined guarantee is $\rho g(\eta)$ only when the rounding theorem's assumptions apply to the generated fractional solution.

    The proof obligation is therefore not a generic variance inequality. A variance bound by itself does not control expected cost, and prediction-dependent sampling does not automatically preserve feasibility or robustness. P3 applies when the literature supplies both an error-sensitive fractional guarantee and a compatible online rounding result. The rounding factor may dominate the final ratio and may prevent exact consistency even when the fractional prediction is perfect.

    \item \emph{\textbf{P4}: Prediction-Access, Portfolio, and Switching Reductions.}
    A portfolio result must connect predictor choice to the online problem's actual cost. For losses in $[0,1]$, a standard Hedge guarantee for weights $w_t$ in the probability simplex has the form
    \begin{equation}
       \sum_t\sum_{j=1}^{M}w_{j,t}\ell_t(\hat y_{j,t})
       \leq \min_j\sum_t\ell_t(\hat y_{j,t})+O(\sqrt{T\log M}),
    \end{equation}
    where $M$ is the number of predictors. The left side is also the expected loss when an expert is sampled from $w_t$. If $\bar y_t=\sum_jw_{j,t}\hat y_{j,t}$ is feasible and $\ell_t$ is convex, Jensen's inequality additionally bounds $\sum_t\ell_t(\bar y_t)$ by this weighted loss. Without convexity, the latter form is not implied. The surrogate-loss guarantee does not by itself imply a competitive ratio: one also needs a reduction from $\ell_t$ to algorithm cost, a compatible state or migration rule, and a benchmark whose scale absorbs any additive regret.

    Anand et al.~\cite{anand2022multiple} give a potential-based framework for multiple predictions in online covering, while Dinitz et al.~\cite{dinitz2022portfolios} provide problem-specific portfolio algorithms and Antoniadis et al.~\cite{antoniadis2023mixing} analyze combinations for metrical task systems against static or switching-aware benchmarks. Cosa and Eli\'a\v{s}~\cite{cosa2025bandit} restrict each round to one queried MTS predictor. Writing $\mathrm{OPT}_{\leq 0}$ for the cost of the best supplied heuristic, their main regret upper bound is $O(\!\mathrm{OPT}_{\leq 0}^{2/3})$ when the metric diameter, number of predictors, and $m$-delayed-access parameter are treated as constants; the lower bound is $\widetilde{\Omega}(\!\mathrm{OPT}_{\leq 0}^{2/3})$, i.e., tight only up to logarithmic factors. Nguyen~\cite{nguyen2026covering} instead compares online covering decisions with a time-varying linear combination of experts. Sparse-query and costly-prediction models also belong to P4 when they control whether or when a predictor is accessed~\cite{drygala2023costly,im2022parsimonious,sadek2024reduced}. These results justify P4 only under their stated feedback, state, query, and comparator models. Robustness follows when the portfolio contains or is coupled to a robust algorithm and the combination theorem preserves its bound; adversarial regret against a pool of non-robust experts is not sufficient.

    \item \emph{\textbf{P5}: Algorithms with Distributional Predictions.}
    Here the prediction itself is a probability distribution rather than a point. Dinitz et al.~\cite{dinitz2024binarysearch} interleave distribution-guided median probes with classical binary search and obtain expected query complexity $O(H(p)+\max\{\log\eta,0\})$, where $p$ is the true access distribution and $\eta$ is its earth mover distance from the prediction. The truncation is essential because $\eta$ may be smaller than one or zero. Angelopoulos et al.~\cite{angelopoulos2024distributional} design contract schedules that use distributional advice while retaining a problem-specific worst-case guarantee.

    These examples do not instantiate a general coherent-risk theorem. A risk-sensitive objective may separately use expectation, CVaR, an entropic risk measure, or a quantile, but these objects have different axioms: entropic risk is convex rather than coherent in general, and a quantile (VaR) is not coherent in general. A Wasserstein bound additionally requires regularity of the loss and distributions, and a bound for a fixed decision does not automatically bound the excess cost of the optimizer selected under a misspecified distribution. P5 therefore records distributional prediction models and their problem-specific analyses, not a universal Lipschitz principle.

\end{enumerate}

\paragraph{Axis B: theorem-level tightness status} E1 is attached when an upper bound $f(\eta)$ is accompanied by a lower bound with the same asymptotic dependence under the same prediction model, error measure, adversary, benchmark, and computational assumptions~\cite{wei2020optimal,antoniadis2023metric,dinitz2024binarysearch,balkanski2024randomized}. Otherwise the result remains at E0 in this survey, which means only that a matching result is not paired here. Logarithmic gaps, class-restricted lower bounds, or changes of benchmark remain E0 under this strict rule. It is not a claim that no matching result exists, nor is E1 a general ranking of empirical validity, breadth, or practical importance. Neither label transfers between incomparable models.

We use ``tight'' below only for explicit matching upper and lower bounds under the same model, not merely for an apparently favorable empirical frontier.

\begin{table}[t!]
\centering
\nomakegapedcells
\setlength{\tabcolsep}{3pt}
\caption{Selected theorem records illustrating the two-axis coding. Each status applies only to the stated guarantee. N/A marks adjacent empirical evidence outside Axis B, not a third evidence level.}
\label{tab:laa_taxonomy}
\fontsize{6.5}{7.25}\selectfont
\begin{tabularx}{\textwidth}{@{}>{\raggedright\arraybackslash}p{2.26cm}>{\centering\arraybackslash}p{0.92cm}YY>{\centering\arraybackslash}p{0.72cm}@{}}
\toprule
\textbf{Theorem record} & \textbf{Axis A} & \textbf{Eligible upper-bound claim} & \textbf{Lower-bound pairing in the same model} & \textbf{Axis B} \\
\midrule
Deterministic ski rental~\cite{purohit2018improving,wei2020optimal}
& P1 & Guarantee curve $(1+\lambda,1+1/\lambda)$ for $\lambda\in(0,1]$ & Asymptotically matching impossibility frontier for the same deterministic model; finite-$B$ rounding remains & E1 \\
\addlinespace[1pt]
Caching with next-arrival predictions~\cite{rohatgi2020near}
& P1 & Error-dependent competitive guarantee under the paper's arrival-error metric & Upper and lower bounds are near-optimal but leave a functional and logarithmic gap under the strict rule & E0 \\
Binary search with distributional predictions~\cite{dinitz2024binarysearch}
& P5 & $O(H(p)+\max\{\log\eta,0\})$ expected queries & Matching dependence under the paper's distributional model & E1 \\
\addlinespace[1pt]
MTS with bandit predictor access~\cite{cosa2025bandit}
& P4 & $O(\mathrm{OPT}_{\leq0}^{2/3})$ regret when auxiliary parameters are constants & $\widetilde\Omega(\mathrm{OPT}_{\leq0}^{2/3})$ leaves logarithmic factors & E0 \\
\addlinespace[1pt]
Stochastic online bidding~\cite{angelopoulos2025bidding}
& P1, P5 & Distributional deterministic: tight Pareto trade-off; randomized point oracle: class-specific upper bounds & The distributional deterministic frontier is matched; the randomized results do not form one blanket match & P5: E1; P1: E0 \\
\addlinespace[1pt]
Randomized strategic facility location~\cite{balkanski2024randomized}
& outside P1--P5 & The line: truthful-in-expectation frontier; higher-dimensional settings: model-specific upper bounds & The line has a matching impossibility result; the other settings do not share one matched frontier & line: E1; others: E0 \\
\addlinespace[1pt]
ALEX adaptive learned index~\cite{ding2020alex}
& adjacent & Workload-dependent empirical update and lookup performance & No eligible consistency, robustness, or error-dependent theorem & N/A \\
\bottomrule
\end{tabularx}
\end{table}

\subsection{Reading the Taxonomy: Where the Gaps Are}
\label{subsec:gaps}

Table~\ref{tab:laa_taxonomy} reports representative positive assignments rather than a mechanism $\times$ domain grid with empty cells. Unlisted combinations remain search and modeling questions, not asserted research gaps. For example, applying P2 to a learned data structure would require an online relaxation and dual invariant that deterministic gapped-array layouts do not presently expose. Applying P5 to a graph problem would require a tractable distributional object and a stability argument linking its misspecification to graph cost. In mechanism design, any P3 or P5 construction must preserve incentive constraints in addition to feasibility.

Other combinations expose accounting questions. A learned-index portfolio would need to charge model-selection latency, rebuild cost, and state migration before regret could imply query latency. Data-driven selection of $\lambda$ must distinguish an offline validation rule from an online rule that observes only causal feedback. These are questions suggested by the included models; the survey makes no claim that the corresponding literatures are empty.

Finally, tight upper and lower bounds are concentrated in a smaller set of well-specified models, including ski rental, distributional binary search, and selected prediction or advice models for caching and strategic facility location. Graph, data-structure, and mechanism-design results often optimize different quantities, e.g., competitive ratio, approximation, query complexity, or running time, so their guarantees should not be ranked on one maturity scale without preserving those distinctions.

\section{Problem Domains}
\label{sec:domains}

Table~\ref{tab:domain_overview} consolidates the five domains while preserving the objective in each row. Evidence status is intentionally omitted from these grouped rows because one row may cite several incomparable theorems. The selected records in Table~\ref{tab:laa_taxonomy} assign E0 or E1 only to a specific stated guarantee.

We now examine five problem domains and one cross-cutting methodological strand. Each domain is presented in the same three parts: the classical formulation, the LAA-theoretic extension, and the representative problem classes. A short assessment of what remains open in that domain follows. The uniformity is deliberate. It lets a reader who cares about one domain read a single subsection, and it lets a reader comparing domains see which structural features drive the differences in guarantee.

\begin{table}[t!]
\nomakegapedcells
\centering
\setlength{\tabcolsep}{3pt}
\caption{Consolidated view of representative problem domains. Mechanisms can fall outside P1--P5. Grouped rows summarize objectives and guarantees but do not receive a single evidence label.}
\label{tab:domain_overview}
\fontsize{6.5}{7.25}\selectfont
\begin{tabular}{@{}
    >{\raggedright\arraybackslash}p{1.56cm}
    >{\raggedright\arraybackslash}p{2.67cm}
    >{\raggedright\arraybackslash}p{2.55cm}
    >{\centering\arraybackslash}p{1.23cm}
    >{\raggedright\arraybackslash}p{4.81cm}
@{}}
\toprule
\textbf{Domain} & \textbf{Problem class} & \textbf{Core tension} & \textbf{Mechanism by cited result} & \textbf{Representative guarantee} \\
\midrule
\multirow{4}{=}{\emph{Online optimization} (\S\ref{subsec:online_opt})}
& One-dimensional and multistate stopping (ski rental, Bahncard, dynamic power management)
& Irrevocable timing vs.\ unknown horizon
& P1 for ski rental; problem-specific extensions
& Ski rental has the asymptotically tight curve $(1{+}\lambda,\,1{+}1/\lambda)$; extensions use different guarantees~\cite{purohit2018improving,wei2020optimal,zhao2024bahncard,antoniadis2021power} \\
\addlinespace[1pt]
& Metrical task systems
& Movement cost vs.\ service cost
& P1 robustification; P4 predictor mixing
& Robust black-box combinations and switching-aware mixtures of predictors~\cite{antoniadis2020metric,antoniadis2023mixing} \\
\addlinespace[1pt]
& Online convex optimization
& Adversarial regret vs.\ temporal regularity
& optimistic online learning
& Classical optimistic methods give regret controlled by $\sqrt{\sum_t \|g_t-\hat g_t\|_*^2}$ under the stated geometry; imperfect-hint refinements use more specific quantities~\cite{rakhlin2013predictable,bhaskara2020hints} \\
\addlinespace[1pt]
& Scheduling, load balancing, online allocation
& Capacity uncertainty vs.\ irrevocable assignment
& P3 for Lattanzi et al.; other mechanisms differ
& Bounds use distinct weight, size, speed, or learned-parameter errors and objectives~\cite{lattanzi2020online,cohen2023allocation,cohen2025allocation,lindermayr2023speed,zhao2026smoothness} \\
\midrule
\multirow{3}{=}{\emph{Caching} (\S\ref{subsec:caching})}
& Prediction-aware paging
& $1$-consistency without prohibitive overhead
& P1
& Error-dependent trade-offs and reduced-prediction variants~\cite{rohatgi2020near,sadek2024reduced} \\
\addlinespace[1pt]
& Randomized robustification
& Preserving $1$-consistency while capping the worst case
& problem-specific robustifier
& $1$-consistent and $(2H_{k-1}{+}2)$-robust with $O(1)$ additional algorithmic time per request in the cited model~\cite{chen2025robustifying} \\
\addlinespace[1pt]
& Weighted and succinct variants
& Non-uniform costs; coarse predictions
& Structural analysis; one-bit advice
& Bounds depending on weight classes and prediction error; consistency, robustness, and smoothness with one-bit advice~\cite{jiang2020weighted,bansal2022weighted,antoniadis2023succinct} \\
\midrule
\multirow{3}{=}{\emph{Data structures} (\S\ref{subsec:data_structures})}
& Indexes with worst-case error bounds
& Query time vs.\ approximation error and data regularity
& problem-specific segmentation
& $\Theta(m)$ space for the minimum number $m$ of $\epsilon$-approximate segments; $m\le n/(2\epsilon)$ for integer $\epsilon\geq1$~\cite{ferragina2020pgm} \\
\addlinespace[1pt]
& Adaptive layouts and model selection
& Static models may degrade under workload drift
& adjacent systems
& Dynamic layouts and workload-dependent empirical performance~\cite{ding2020alex,wongkham2022updatable} \\
\addlinespace[1pt]
& Filters and sketches
& False-positive rate vs.\ space
& adjacent systems
& Distribution-dependent false-positive/space trade-offs with a backup filter~\cite{mitzenmacher2018sandwiching,hsu2019frequency} \\
\midrule
\multirow{3}{=}{\emph{Graphs and networks} (\S\ref{subsec:graph})}
& Incremental shortest paths
& Running time vs.\ error in a predicted update sequence
& prediction-guided dynamic update
& $(1+\epsilon)$-approximate distances with update time depending on sequence error~\cite{mccauley2025sssp} \\
\addlinespace[1pt]
& Matching, flow, structure
& Dual initialization sensitivity
& P2 for learned duals; problem-specific advice otherwise
& Runtime warm starts from learned optimization solutions; error-dependent online guarantees~\cite{dinitz2021duals,sakaue2022dca,lavastida2021instancerobust,davies2023flows} \\
\addlinespace[1pt]
& Dynamic maintenance
& Error in predicted update order
& sequence prediction
& Prediction-guided incremental and dynamic maintenance under distinct error models~\cite{mccauley2023listlabeling,mccauley2024topological,agarwal2024submodular} \\
\midrule
\multirow{3}{=}{\emph{Mechanism design} (\S\ref{subsec:mech_design})}
& Pricing and revenue
& Prediction error vs.\ revenue guarantees
& problem-specific
& Error-dependent revenue guarantees for learned reserve information~\cite{medina2017revenue} \\
\addlinespace[1pt]
& Online selection and matching
& Irrevocable decisions; arrival-model ambiguity
& problem-specific
& Advice-sensitive and fairness-aware bounds under distinct secretary, chosen-order, and matching models~\cite{dutting2021secretaries,jin2022bipartite,balkanski2024fair,karisani2026secretary} \\
\addlinespace[1pt]
& Facility location; general allocation
& Incentives, welfare, and revenue under error
& Facility: problem-specific; allocation: weakest competitor
& Facility-location guarantees use stated prediction models~\cite{agrawal2022facility,balkanski2024randomized}; Prasad et al. give bicriteria welfare/revenue guarantees from side information~\cite{prasad2023sideinfo} \\
\bottomrule
\end{tabular}
\end{table}

\subsection{Online Optimization}
\label{subsec:online_opt}

\subsubsection{Classical Formulation}
Online optimization models sequential decision-making under uncertainty. An online algorithm receives a request sequence $\sigma = (\sigma_1, \ldots, \sigma_T)$ and chooses $x_t$ using only the revealed history. Feasibility, state transitions, and movement costs may couple decisions across time. For a randomized cost-minimization algorithm, the competitive ratio is
\begin{equation}
    \mathrm{CR}(\mathcal{A}) = \sup_\sigma
    \frac{\mathbb E[\mathrm{cost}_{\mathcal A}(\sigma)]}{\mathrm{OPT}(\sigma)},
\end{equation}
where $\mathrm{OPT}$ obeys the same feasibility and state constraints but knows the full sequence. An additive constant is included when required by the problem's standard definition. Ski rental has deterministic optimum ratio $2$ and randomized optimum ratio $e/(e-1)\approx1.58$~\cite{karlin1988snoopy,karlin1994competitive}.

\subsubsection{LAA Extension}
LAA adds a prediction and states separately what happens when it is exact and when it is arbitrary. Robustness is a finite prediction-independent ratio, not necessarily the optimal prediction-free ratio. For deterministic ski rental the standard guarantee trade-off is $(1+\lambda,1+1/\lambda)$ for $\lambda\in(0,1]$, and its impossibility frontier is asymptotically matching as $B$ grows~\cite{purohit2018improving,wei2020optimal}. Consequently, approaching exact $1$-consistency sends this deterministic robustness bound to infinity; the randomized frontier is different and should not be conflated with the deterministic formula.

\subsubsection{Representative Problem Classes}

\paragraph{One-Dimensional Stopping Problems}
The simplest online problems involve a single irrevocable decision whose optimal timing depends on an unknown scalar. In deterministic ski rental, the prediction determines whether the purchase threshold is $\lceil\lambda B\rceil$ or $\lceil B/\lambda\rceil$; $\lambda$ is fixed before the realized error is known. The classical randomized benchmark is $e/(e-1)$-competitive, not $e$-competitive. Deterministic and randomized trade-off frontiers are analyzed separately~\cite{purohit2018improving,gollapudi2019rentorbuy,wei2020optimal}.

This stopping framework generalizes to problems with richer cost structures. The \emph{Bahncard problem} models travelers choosing between purchasing time-limited discount cards or paying full fare, introducing expiration and residual-value effects absent in ski rental. Zhao et al.~\cite{zhao2024bahncard} analyze consistency--robustness trade-offs for this model. Dynamic power management with multiple power-saving states is another generalization: the prediction concerns the length of an idle period, and the guarantee degrades with its error~\cite{antoniadis2021power}. Contract scheduling changes the objective and admits both point and distributional advice~\cite{angelopoulos2023contract,angelopoulos2024distributional}. Online bidding admits several stochastic formulations that must be separated~\cite{angelopoulos2025bidding}. The deterministic strategy under distributional predictions has a matched Pareto frontier. For a randomized algorithm receiving a single-valued oracle, the paper instead gives upper bounds for a parameterized strategy class and a distinct general lower bound; these do not establish one blanket matching frontier. Tail-risk formulations evaluate a quantile of the competitive ratio rather than its expectation~\cite{dinitz2024tailrisk}. General analyses also show that consistency, robustness, smoothness, and stochastic performance can impose distinct, interacting trade-offs~\cite{benomar2025tradeoffs}.

\paragraph{Stateful Decision-Making on Metric Spaces}
Unlike stopping problems, \emph{metrical task systems} (MTS) require maintaining and transitioning among states in a metric space $(X, d)$ while serving requests. A central trade-off is between \emph{movement cost} (transitioning between states) and \emph{service cost} (serving a request from the current state). On arbitrary $n$-point metrics, deterministic classical guarantees scale linearly with $n$; randomized guarantees obey different bounds, so the adversary and randomization model must be stated explicitly~\cite{antoniadis2020metric}.

Prediction models for metrical task systems typically provide recommended states, actions, or several predictors rather than generic ``future locations''. Antoniadis et al.~\cite{antoniadis2020metric} give robust black-box combinations for online metric algorithms. Their later work~\cite{antoniadis2023mixing} studies dynamic combinations of multiple predictors, obtaining an $O(\ell^2)$ ratio against an unconstrained combination of $\ell$ predictors and a $(1+\epsilon)$ guarantee against a benchmark with a slightly constrained number of switches. With bandit access, only one predictor can be queried per round and the unqueried movement-dependent costs are not observed. Cosa and Eli\'a\v{s}~\cite{cosa2025bandit} obtain $O(\!\mathrm{OPT}_{\leq0}^{2/3})$ regret relative to the best supplied heuristic when the metric diameter, number of predictors, and $m$-delayed-access parameter are constants. Their $\widetilde\Omega(\!\mathrm{OPT}_{\leq0}^{2/3})$ lower bound leaves logarithmic factors, so Table~\ref{tab:laa_taxonomy} records E0 under the strict same-asymptotics rule. These are problem-aware combination results, not the generic predicted-dual update of P2 or a universal $O(\log\eta)$ frontier.

\paragraph{Convex Optimization with Predictable Structure}
\emph{Online convex optimization} (OCO) generalizes many online problems by allowing convex loss functions $f_t$ revealed after each decision. Under the usual bounded-domain and bounded-gradient assumptions, classical OCO achieves $O(\sqrt{T})$ regret against adversarial sequences, but this horizon-only rate can be pessimistic when losses exhibit temporal regularity.

When gradient or loss predictions are available, classical optimistic OMD/FTRL can replace a horizon-only bound by one depending on cumulative hint error, such as a square root of $\sum_t\|g_t-\hat g_t\|_*^2$ under the theorem's geometry and regularity assumptions~\cite{rakhlin2013predictable}. Bhaskara et al.~\cite{bhaskara2020hints} study imperfect hints in a constrained Hilbert-space setting, obtaining refined upper bounds and near-matching lower bounds for their own error quantities; their result should not be cited as the source of every optimistic-gradient formula. These are point-prediction regret guarantees, not P5 distributional optimization, and perfect hints do not imply constant regret without checking the precise bound and comparator assumptions.

\paragraph{Resource Allocation and Scheduling}
Problems such as \emph{load balancing}, \emph{bin packing}, and \emph{scheduling} involve assigning jobs to resources with limited capacity. Classical worst-case algorithms do not use predicted workload structure, whereas a prediction-driven policy without an independent guarantee can incur high cost when its predictions are inaccurate.

In these models, LAA results use problem-specific prediction objects and analyses. Learned weights can guide a fractional restricted-assignment schedule whose $O(\log\eta)$ guarantee is followed by an online rounding factor~\cite{lattanzi2020online}. Im et al.~\cite{im2021nonclairvoyant} analyze total-completion-time scheduling with predicted job sizes under a tailored error measure. Cohen and Panigrahi study broader online allocation frameworks with learned per-agent parameters or weights, first for divisible allocation and later for a more general option-based model~\cite{cohen2023allocation,cohen2025allocation}; these formulations and objectives should not be collapsed into one scheduling theorem. Other work predicts machine speeds, job sizes, rankings, item frequencies, or deadlines and uses different distortion measures~\cite{lindermayr2023speed,balkanski2023speed,balkanski2023energy,azar2021flowtime,azar2022distortion,im2021knapsack,ahmadian2023loadbalancing,dinitz2026jrp}. Zhao and Zomaya~\cite{zhao2026smoothness} study a refined smoothness definition for nonclairvoyant scheduling; its bounds are tied to their information restriction and are not interchangeable with Definition~\ref{def:smoothness}. For nonclairvoyant joint replenishment with deadlines, Dinitz et al.~\cite{dinitz2026jrp} prove $O(\min\{\eta^{1/3}\log^{2/3}n,\sqrt{\eta},\sqrt n\})$ and an $\Omega(\eta^{1/3})$ lower bound only for a specified class of deterministic algorithms. The logarithmic gap and class restriction make this near-tight evidence, not E1. For online packing with concave objectives, Grigorescu et al.~\cite{grigorescu2025packing} switch between advice and a classical algorithm; for convex covering, they extend prediction-guided primal--dual methods. Another line combines multiple experts against a time-varying covering benchmark~\cite{nguyen2026covering}. These results should therefore be compared only after matching the prediction object, objective, feedback, and error metric.

\paragraph*{Open gaps in this domain}
Two structural questions stand out in the surveyed corpus. The ski-rental model has an asymptotically tight consistency--robustness frontier, whereas many higher-dimensional resource-allocation and scheduling results summarized here are upper bounds under distinct prediction models. Separately, trust-parameterized results choose $\lambda$ before the relevant error is known. Portfolio methods, formal robustifiers, and fallbacks can adapt in particular models, but their regret, state-coupling, and overhead costs must be included in the analysis.

\subsection{Caching and Paging}
\label{subsec:caching}

\subsubsection{Classical Formulation}
The caching (or paging) problem is a cornerstone of online algorithm design, where a decision-maker must serve a sequence of $T$ page requests with a cache of limited size $k$, aiming to minimize cache misses. In the deterministic setting, the competitive ratio is lower bounded by $k$, a barrier matched by LRU and FIFO~\cite{sleator1985amortized}. Against an oblivious adversary, randomized paging has lower bound $H_k$, and the Marker algorithm is $(2H_k-1)$-competitive~\cite{fiat1991competitive}. Partition and Equitable achieve the optimal $H_k$ ratio, with different computational costs~\cite{mcgeoch1991strongly,achlioptas2000competitive}. The offline optimum, Belady's rule, evicts the page whose next request is furthest in the future~\cite{belady1966study}.

\subsubsection{LAA Extension}
Learning-augmented caching extends the classical model by incorporating machine-learned predictions to guide eviction decisions. Predictions typically take one of two forms:
\begin{itemize}
    \item \emph{Next Request Time (NRT)}: When page $p$ is requested, the predictor supplies a timestamp $\hat{t}_p$ for its next request; the algorithm retains that prediction while $p$ remains cached. Accurate NRT predictions approximate the information used by Belady's oracle.
    \item \emph{Binary or Action Predictions}: A binary label indicating whether a page will be requested again within a certain window, or a direct eviction recommendation. These are often easier to learn but provide coarser information.
\end{itemize}
An important question is whether \emph{$1$-consistency}, i.e., matching Belady's oracle when predictions are perfect, can coexist with robustness and low algorithmic overhead. Blindly following NRT predictions can be $1$-consistent but have unbounded worst-case performance under adversarial errors. Reverting entirely to a classical algorithm ignores the prediction~\cite{lykouris2018competitive,lykouris2021competitive,rohatgi2020near}. Later constructions analyze different combinations of consistency, robustness, prediction use, and computational overhead~\cite{wei2020better}.

\subsubsection{Representative Problem Classes}

\paragraph{Deterministic Switching Approaches}
Combining a prediction-following paging algorithm with a classical algorithm is difficult because their cache states need not agree and moving between states can itself cause misses. A valid switching theorem must charge that coupling cost; one cannot set $\lambda=\lambda(\eta)$ unless the relevant error has already been revealed. Separate work studies how to retain guarantees while querying the predictor less often~\cite{im2022parsimonious,sadek2024reduced}. Query complexity, running time, and competitive ratio are distinct resources and should be reported separately.

\paragraph{Randomized Robustification}
For a relaxed-Belady-following (RB-following) caching algorithm $A$, \textsc{Guard}$\&A$ preserves $1$-consistency and achieves $(2H_{k-1}+2)$-robustness while maintaining the base algorithm's asymptotic running time~\cite{chen2025robustifying}. The paper establishes RB-following instantiations for BlindOracle, LRB, and Parrot, so the corresponding guarded variants receive the stated guarantees in its model. The result relies on paging-specific phase and protection rules; it is not a domain-independent wrapper, and its overhead accounting excludes predictor inference.

\paragraph{Structured and Weighted Extensions}
Beyond uniform-cost paging, weighted paging assigns known heterogeneous eviction costs to pages. The prediction studied by Bansal et al.~\cite{bansal2022weighted} concerns each page's next arrival, not its weight. Their deterministic and randomized guarantees depend on the number of distinct weight classes and degrade with prediction error toward the classical $k$ and $O(\log k)$ regimes; the analysis uses structural properties of Belady's algorithm and prediction-aware potential functions. Jiang et al.~\cite{jiang2020weighted} instead compare PRP, strong-lookahead, SPRP, Static, and Follow models through model-specific algorithms, error measures, and lower bounds. Their contribution is not classified here as a prediction-guided primal--dual update.

Paging with succinct predictions instead limits the advice accompanying each request to one bit~\cite{antoniadis2023succinct}. The paper proves consistency, robustness, and smoothness guarantees together with lower bounds under that one-bit model. It does not assume a Markov chain or aggregate Markov predictors of different orders.

\paragraph*{Open gaps in this domain}
Three questions remain. First, aggregate per-request error can hide whether error is concentrated on frequently requested objects. Second, the fixed-cache results reviewed here do not model co-located services whose available capacity varies over time. Third, additional algorithmic running time excludes predictor-call cost. Query-budget work counts how often a predictor is used~\cite{im2022parsimonious,sadek2024reduced}, but converting inference latency or energy into the caching objective requires a different model.

\subsection{Learned Data Structures}
\label{subsec:data_structures}

\subsubsection{Classical Formulation}
Learned data structures study how data regularity can improve practical search, indexing, and filtering, with formal guarantees or empirical evidence depending on the design. Given a sorted dataset $\{(x_i, y_i)\}_{i=1}^n$, a comparison search locates a key in $O(\log n)$ comparisons, while hash-based alternatives can provide expected $O(1)$ lookup without ordered range-query semantics~\cite{richter2015seven}. Learned indexes use a fitted model to narrow the search interval; any improvement is conditional on data regularity and must include the model and correction structure in the space and time accounting.

\subsubsection{LAA Extension}
A learned index uses a model $f_\theta:x\mapsto\widehat{\mathrm{pos}}(x)$ to narrow the search range. When a certified maximum position error $\epsilon_\theta$ is available, a local binary search costs $O(\log(\epsilon_\theta+1))$; the $+1$ covers exact predictions. Total lookup cost also includes model space and inference time~\cite{kraska2018case}. These quantities form a workload- and architecture-dependent trade-off, not a universal Pareto frontier.

Unlike online algorithms with irrevocable decisions, learned indexes primarily improve query or update time. A certified bound $\epsilon_\theta\le\bar\epsilon$ is conditional on the fitted structure; it is not robustness to arbitrary prediction error. A genuine worst-case fallback searches a larger range or the full index, typically recovering a classical bound such as $O(\log n)$. We therefore treat learned indexes as an adjacent prediction-sensitive literature and keep their running-time guarantees distinct from competitive ratios.

\subsubsection{Representative Problem Classes}

\paragraph{Indexes with Worst-Case Error Guarantees}
Some early learned-index designs emphasized average-case performance. Latency-sensitive use additionally requires a certified or monitored error bound and a correction or fallback procedure. One formal objective is to ensure $\epsilon_\theta \leq \bar{\epsilon}$ uniformly while controlling space.

The PGM-index computes an optimal piecewise-linear approximation for a fixed array and integer error tolerance $\epsilon\geq1$: if $m$ is the minimum number of $\epsilon$-approximate segments, the recursive index uses $\Theta(m)$ space, and $m\le n/(2\epsilon)$ is a worst-case upper bound~\cite{ferragina2020pgm}. The minimum-segment result is different from claiming that $\Theta(n/\epsilon)$ space is necessary for every dataset. The same paper also gives a fully dynamic variant with insertion and deletion bounds. FITing-tree uses linear models and a fixed integer error parameter $\epsilon$ indexed by a B-tree~\cite{galakatos2019fiting}; it is not a distributional-risk method and does not adapt a local tolerance through P5.

\paragraph{Adaptive Layouts and Model Selection}
Static learned models may degrade when access patterns or data distributions drift over time. The challenge is adapting the index structure while accounting for retraining and data-movement costs.

ALEX is an updatable learned index using learned linear models, gapped-array data nodes, cost models, and adaptive splitting or expansion policies~\cite{ding2020alex}. Gapped arrays are a deterministic layout technique rather than randomized rounding, and ALEX does not prove a Hedge-style guarantee over linear and polynomial experts. Wongkham et al.~\cite{wongkham2022updatable} benchmark updatable learned indexes under multiple workloads; they do not propose an ensemble that routes each query among specialized indexes. These systems provide empirical adaptability, not a consistency--robustness theorem under distribution shift.

\paragraph{Filters and Emerging Representations}
Beyond ordered indexes, learned techniques have been applied to approximate membership queries and frequency estimation, with guarantees that differ from those of learned indexes.

Learned Bloom filters combine a classifier with a backup Bloom filter for keys the classifier might falsely reject~\cite{mitzenmacher2018sandwiching,mitzenmacher2021algorithms}. Their false-positive rate depends on the classifier, the negative-query distribution, and the memory allocated to the backup. Even a classifier with perfect precision does not by itself force the total false-positive rate to zero, because the backup Bloom filter can return false positives. Nor is ``no worse than a classical Bloom filter'' automatic for every allocation; it must be established under the paper's distributional and space assumptions.

\paragraph*{Open gaps in this domain}
Static learned models expose one set of guarantees, while PGM and ALEX show that updates can also be supported with formal or empirical costs~\cite{ferragina2020pgm,ding2020alex}. Benchmark studies find that performance varies across insertion patterns and datasets~\cite{wongkham2022updatable}. Open questions concern adversarial update sequences, rebuild policy, and end-to-end accounting for retraining and memory movement. A model portfolio would additionally need a state-migration analysis before per-query selection regret could imply an index-level guarantee.

\subsection{Graph Algorithms and Networking}
\label{subsec:graph}

\subsubsection{Classical Formulation}
Graph optimization encompasses problems such as shortest paths, matchings, cuts, and spanning trees, with applications in transportation, communication networks, and social network analysis. For an edge-weighted minimization problem, a common abstraction is $\min_{S \in \mathcal{F}} \sum_{e \in S} w(e)$ over a problem-specific feasible family $\mathcal{F}$; maximization and flow formulations use different objectives. Because feasibility and optimality can depend nonlocally on the graph, a local prediction error may alter a global solution. Prediction-augmented work therefore asks, model by model, whether a declared hint can improve running time or online performance while retaining a stated worst-case guarantee.

\subsubsection{LAA Extension}
Prediction-augmented graph work studies several distinct objectives. In offline matching and flow, a predicted dual solution can improve running time while the output remains exact. In incremental graph problems, a predicted update sequence can warm-start a data structure and reduce total update time. In online matching, flow allocation, or load balancing, advice can improve competitive performance. ``Consistency'' must therefore name the improved quantity, e.g., running time, approximation, or online cost, rather than uniformly mean recovery of an exact optimal subgraph. Prediction errors likewise range from dual distance to inversions or edit distance between update sequences and are not interchangeable.

\subsubsection{Representative Problem Classes}

\paragraph{Path and Routing Problems}
Shortest-path problems exhibit global dependencies, but the representative incremental result here predicts update order rather than a shortest path or hidden edge weights.

McCauley et al.~\cite{mccauley2025sssp} consider incremental approximate single-source shortest paths. Before edges arrive, the data structure receives a prediction of the entire edge-insertion sequence and uses it to warm-start its state. The result maintains $(1+\epsilon)$-approximate distances with total running time that degrades with sequence-prediction error and recovers a worst-case bound up to logarithmic factors. It is not a switch between a predicted path and Dijkstra, and its guarantee concerns dynamic-data-structure running time. In a separate offline model, Polak and Schmidt~\cite{polak2026apsp} predict detour-certificate vertex sets for all-pairs shortest paths and obtain running time $O(n^{2.83}+\eta n)$, where $\eta$ counts pairs whose prediction does not suffice to compute and certify the shortest-path distance. This conditional speedup is neither a dynamic update guarantee nor a consistency--robustness ratio.

\paragraph{Matching, Flow, and Structural Optimization}
LP duality supplies certificates for several matching and flow formulations, making predicted dual information a plausible source of computational speedups. The challenge is translating an inaccurate structural prediction into a valid initialization or certificate without violating feasibility.

P2 takes a distinctive form here. Dinitz et al.~\cite{dinitz2021duals} use predicted dual information to accelerate matching, with running time that depends on distance from suitable optimal duals. Related prediction-sensitive ideas appear in flows~\cite{davies2023flows} and online fractional matching~\cite{choo2025fractional}, but their prediction objects and guarantees differ. In the offline matching and flow results, the improved quantity is running time rather than solution quality. Any use of a predicted dual must also respect or repair feasibility under the conditions of the specific algorithm; there is no single warm-start theorem covering all the cited online and offline problems.

Lavastida et al.~\cite{lavastida2021instancerobust} instead design learnable, instance-robust predictions for online matching, flow allocation, and load balancing. Their prediction objects and reductions are problem specific; the paper does not establish a generic prediction-biased edge-sampling paradigm.

Predictions of imperfect and unknown quality have also been analyzed for online bipartite matching~\cite{choo2024imperfect}; predicted structure or sequences have been used for minimum cut, online Steiner tree, and acknowledgement problems~\cite{moseley2025mincut,moseley2022steiner,im2023acknowledgement}. These works do not share a single mechanism. Discarding a prediction and running a classical method is safe only when its additional computation, prior irreversible decisions, and state conversion are included in the model.

\paragraph{Dynamic and Distributional Settings}
Graph updates and uncertain future demands require models that specify when predictions are issued, how sequence error is measured, and whether the update process may react to the algorithm. A static prediction theorem is not automatically invalid under change, but it may omit update and feedback costs.

Prediction-guided data structures maintain solutions under update sequences for problems including list labeling and topological ordering~\cite{mccauley2023listlabeling,mccauley2024topological}. Dong et al.~\cite{dong2025correlation} instead study arbitrary-order streaming correlation clustering with predicted pairwise distances. Their guarantees couple approximation and space, e.g., a better-than-$3$ approximation on complete graphs under good prediction quality with $\widetilde O(n)$ space, rather than a dynamic-update competitive ratio. Dynamic submodular maximization provides another objective: predicted insertion and deletion times yield amortized update time polynomial in logarithms of the timing error and model parameters while maintaining a $(1/2-\epsilon)$ approximation in expectation~\cite{agarwal2024submodular}. Classical streaming and dynamic bounds vary by problem, approximation, adversary, and update model, so these statements are not interchangeable. Multi-predictor guarantees also require a state-coupling theorem rather than generic expert routing.

Classical route-planning algorithms~\cite{bast2016route} and empirical learned routing are adjacent to, but distinct from, a formal P5 treatment of predicted traffic distributions. The corpus summarized here does not supply an end-to-end consistency--robustness theorem for that formulation. Such a result would require a precise routing model, prediction object, distributional error measure, and benchmark.

\paragraph*{Open gaps in this domain}
Among the representative graph results in Table~\ref{tab:domain_overview}, matching lower bounds are sparse, so an upper bound should not be read as a complete characterization unless its cited model supplies one. Structural predictions also raise a question that scalar point predictions avoid: what does it mean for a predicted \emph{combinatorial object} to be slightly wrong? Symmetric difference and edit distance are both used, but they are not generally comparable and neither is universally the right error measure. For learned routing, a useful next step is therefore to state one concrete model in which accurate traffic advice improves the objective while adversarial advice retains a bounded approximation guarantee.

\subsection{Mechanism Design and Algorithmic Game Theory}
\label{subsec:mech_design}

\subsubsection{Classical Formulation}
Mechanism design studies allocations and payments when agents hold private information. Objectives include welfare, revenue, cost, and makespan, and incentive notions must be distinguished: dominant-strategy incentive compatibility (DSIC) requires truthful reporting for every report profile, whereas Bayesian incentive compatibility (BIC) is defined in expectation over other agents' types. Myerson's result concerns a particular Bayesian single-parameter auction model~\cite{myerson1981optimal}; it is not a template for all mechanism-design problems. For maximization objectives, approximation is reported as $\mathrm{OPT}/\mathbb E[\mathrm{ALG}]$ or an equivalent factor.

\subsubsection{LAA Extension}
Learning-augmented mechanism design uses predictions of values, processing times, locations, or side information while preserving a specified incentive notion. The model is not limited to predicted type distributions. The problems include, e.g., unrelated-machine scheduling. Balkanski et al.~\cite{balkanski2023strategyproof} use predicted processing times and give a strategyproof mechanism with consistency and robustness guarantees; this is not an auction-reserve result. Sample-based mechanisms can also be DSIC or universally truthful, so sampling does not imply that IC holds only in expectation. Every result must state separately the prediction object, approximation objective, error measure, and incentive concept.

\subsubsection{Representative Problem Classes}

\paragraph{Auctions and Pricing Problems}
Auction revenue can be sensitive to distributional assumptions and reserve prices. A learning-augmented mechanism must therefore state how revenue changes under prediction error while preserving its declared incentive property.

Medina and Vassilvitskii~\cite{medina2017revenue} relate learned reserve-price information to revenue. Combining mechanisms preserves DSIC only when the selection rule and payments are designed so that an agent cannot manipulate the branch; ordinary switching is not sufficient. Prasad et al.~\cite{prasad2023sideinfo} take a prior-free perspective and build a meta-mechanism using side information and a weakest-competitor construction to obtain welfare and revenue guarantees. Their result is not a predicted-virtual-value primal--dual warm start. Online bidding is non-strategic in its standard target-search formulation, but it is adjacent to this domain: Angelopoulos and Simon~\cite{angelopoulos2025bidding} prove a tight deterministic frontier for distributional predictions and separate randomized results for a single-valued oracle. Those bounds should not be conflated with one another or read as auction truthfulness guarantees.

\paragraph{Online Selection and Matching}
Secretary, prophet-inequality, and online bipartite-matching models all involve irrevocable choices, but their classical benchmarks differ: the secretary success probability is $1/e$, the basic prophet-inequality factor is commonly $1/2$, and online bipartite matching has the $1-1/e$ benchmark in its standard randomized model. They should not be summarized as one $1/e$ barrier.

Advice-augmented secretary and online matching papers establish guarantees in their own models~\cite{dutting2021secretaries,jin2022bipartite}. Recent secretary variants add fairness of selection~\cite{balkanski2024fair} or allow the decision-maker to choose the arrival order from predicted values~\cite{karisani2026secretary}; neither guarantee transfers to a different arrival model. Thresholding is not the same as randomized rounding, and portfolio results do not justify a general claim that multi-item mechanisms use Hedge over bidder observations. Where strategic agents are present, truthfulness must be proved in addition to the online performance bound.

\paragraph{Contracts and Information Design}
Contract design and Bayesian persuasion involve hidden action or hidden information, and the principal's utility depends on agent responses to incentives. Prediction-dependent objectives must therefore be analyzed together with the relevant incentive and best-response constraints.

P5 offers one possible lens: a principal could optimize worst-case expected utility over an ambiguity set around $\hat{\mathcal{D}}$. This construction should not be inferred from robust Bayesian persuasion, which studies uncertainty about information available to receivers and equilibrium selection under its own payoff criterion~\cite{dworczak2022robust}. The feasibility comparison also depends on the incentive model. For a DSIC auction with a distribution-independent feasible mechanism class, perturbing $\hat{\mathcal{D}}$ changes the objective but not the DSIC constraints; under Bayesian incentive constraints, the distribution can also enter feasibility. In a contract, the agent's best response is itself a function of the contract, so a perturbation can change which actions are induced. Stability of the induced action under distributional perturbation is therefore one quantity that an end-to-end theorem would have to control. We did not identify such a P5 theorem in the surveyed corpus and record this combination as open.

\paragraph*{Open gaps in this domain}
The strategic setting differs from many non-strategic models because prediction error need not be exogenous. An agent who understands the predictor may manipulate its inputs, so $\eta$ can become an equilibrium object rather than a fixed parameter. Moreover, an ``instance-optimal'' benchmark is delicate once the comparison class must itself satisfy an incentive constraint. Randomized strategic facility-location results illustrate why evidence status must remain model-specific: they establish upper and lower bounds for stated prediction types and objectives, not a domain-wide frontier~\cite{balkanski2024randomized}. Other reported consistency--robustness curves should be read as achievability results unless the cited work proves matching tightness.

\subsection{Cross-Cutting Methodological Developments}
\label{subsec:methodology_crosscutting}

\subsubsection{Classical Formulation}
Much of the theory conditions on advice $\hat y$. Unless modeled explicitly, this leaves predictor training, query frequency, and changes in the observation--outcome relationship outside the theorem.

\subsubsection{LAA Extension}
Cross-cutting work treats prediction generation, query budget, and uncertainty as design variables. The guarantee can therefore depend on how advice is produced, accessed, and represented, not only on the consuming algorithm.

\subsubsection{Representative Methodological Threads}
We describe four established methodological threads and two directions whose current guarantees remain conditional or incomplete.

\paragraph{Learning the prediction, not just consuming it}
A generic loss need not align with downstream cost because equal prediction errors can perturb different decisions. Decision-focused training differentiates through a relaxation or surrogate~\cite{wilder2019end}; algorithm-aware learning derives sample-complexity bounds for the prediction consumed by a fixed LAA~\cite{khodak2022learning}; explicit-predictor algorithms expose and update the learning rule in stated caching and scheduling models~\cite{elias2024explicit}.

\paragraph{Budgeting predictor queries}
Query-budget models optimize consultations separately from accuracy. Parsimonious caching queries a vanishing fraction of requests, and related work reduces predictions for caching and MTS~\cite{im2022parsimonious,sadek2024reduced}. Drygala et al.~\cite{drygala2023costly} instead attach a cost to prediction in ski rental and Bahncard and analyze whether and when to query. Section~\ref{subsec:overhead} introduces a separate abstract conversion into objective units.

\paragraph{Portfolios in place of a single predictor}
Problem-specific multiple-prediction and portfolio frameworks select among predictors under stated covering, regret, or competitive guarantees~\cite{anand2022multiple,dinitz2022portfolios}. This P4 control loop must include state-transition costs when decisions carry state.

\paragraph{Using uncertainty-quantified predictions}
Sun et al.~\cite{sun2024uncertainty} study ski rental and online search when a prediction carries a range and its stated coverage probability, and formulate a multi-instance learning problem. Gupta et al.~\cite{gupta2022epsilon} instead assume independently $\epsilon$-accurate predictions for caching and covering problems. Neither model calibrates arbitrary confidence scores or makes a coverage statement a generic downstream cost bound.

\paragraph{Two incompletely resolved directions}
Time-decay moment estimation proves results for a suffix-compatible heavy-hitter oracle and empirically instantiates it with sketches, recurrent models, and language models~\cite{nagawanshi2026moment}; it does not prove the oracle premise for a particular language model. A second direction combines distribution-free calibration with P5, but still needs a reduction from coverage to downstream cost.

\paragraph*{Open gaps in this domain}
Relating differentiable surrogates to a discrete objective remains problem specific. Our search found no general lower-bound characterization for query budgets. Portfolio regret converts to a competitive statement only through problem-specific relations between comparator loss and offline optimum.

\section{System-Level Implications and Composition}
\label{sec:las_architecture}

The results reviewed above usually analyze one algorithm, one prediction interface, and one objective. Deployed systems may connect several such components while sharing compute, state, and feedback. We use \emph{learning-augmented systems} (LAS) only as a compact name for examining those interfaces. The perspective is survey-derived: it exposes costs and assumptions that component theorems legitimately leave outside their scope, but supplies no guarantee by architecture alone.

\subsection{Three Considerations Often Outside Component Objectives}
\label{subsec:broken}

The opening scenarios isolate three omissions that a deployment must handle even when a component theorem is correct.

\paragraph{S1: predictions are not free}
Standard competitive objectives normally do not charge inference, training, or sample collection. Thus ``$O(1)$ overhead per request'' describes the algorithm around the prediction, not the cost of producing it. Costly-prediction and query-budget results make different parts of this resource explicit~\cite{drygala2023costly,im2022parsimonious,sadek2024reduced}; Section~\ref{subsec:overhead} discusses an abstract accounting in the objective's units.

\paragraph{S2: prediction quality is only partly observable and can be endogenous}
Robust competitive analysis may already cover adversarial request sequences; the separate issue is observability. A smoothness bound can depend on a target revealed only after the decision or censored by it. In admission, pricing, recommendation, and routing, actions may also change later training data. Such endogeneity requires a problem-specific feedback model, not a generic appeal to distribution shift.

\paragraph{S3: components are not isolated}
Adjacent stages can see different induced instances and different offline benchmarks. Consequently, component ratios alone do not determine a pipeline ratio; benchmark distortion, state transitions, and error sensitivity must also be related.

\subsection{A Reference Architecture}
\label{subsec:refarch}

Figure~\ref{fig:las_arch} is an accounting diagram rather than a prescribed architecture. It separates the decision path from feedback whose signals may arrive at different times.

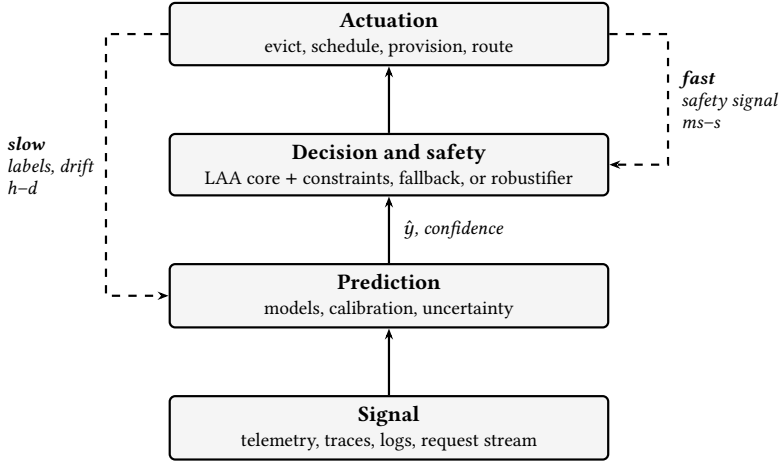
\begin{figure}[t!]
\centering
\resizebox{0.75\textwidth}{!}{%
\begin{tikzpicture}[
  font=\footnotesize,
  layer/.style={draw, rounded corners=2pt, minimum width=5.6cm, minimum height=0.78cm, align=center, thick},
  lbl/.style={font=\scriptsize\itshape, align=left},
  ar/.style={-{Stealth[length=4pt]}, thick},
  fb/.style={-{Stealth[length=4pt]}, thick, dashed}
]
\node[layer, fill=black!4] (sig) at (0,0) {\textbf{Signal}\\[-1pt]\scriptsize telemetry, traces, logs, request stream};
\node[layer, fill=black!4, above=0.85cm of sig] (pred) {\textbf{Prediction}\\[-1pt]\scriptsize models, calibration, uncertainty};
\node[layer, fill=black!4, above=0.85cm of pred] (dec) {\textbf{Decision and safety}\\[-1pt]\scriptsize LAA core $+$ constraints, fallback, or robustifier};
\node[layer, fill=black!4, above=0.85cm of dec] (act) {\textbf{Actuation}\\[-1pt]\scriptsize evict, schedule, provision, route};

\draw[ar] (sig) -- (pred);
\draw[ar] (pred) -- node[right=1pt, lbl] {$\hat y$, confidence} (dec);
\draw[ar] (dec) -- (act);

\draw[fb] (act.east) -- ++(0.75,0) |- node[pos=0.25, right=1pt, lbl] {\textbf{fast}\\ safety signal\\ ms--s} (dec.east);
\draw[fb] (act.west) -- ++(-0.8,0) |- node[pos=0.25, left=1pt, lbl] {\textbf{slow}\\ labels, drift\\ h--d} (pred.west);
\end{tikzpicture}}
\caption{An accounting view of a learning-augmented deployment. Solid arrows form the decision path and dashed arrows denote feedback. Fast mechanisms can use only immediately observable signals; delayed outcomes may support recalibration or retraining.}
\label{fig:las_arch}
\Description{A four-layer stack drawn bottom to top: Signal, Prediction, Decision and safety, and Actuation, connected by solid arrows. A fast dashed feedback path returns to the decision layer and a slower dashed path returns to the prediction layer.}
\end{figure}

The prediction layer must declare the output type and any calibrated uncertainty used by the decision theorem. The decision layer contains the augmented algorithm plus separately justified feasibility checks, constraints, fallbacks, or formal robustifiers. Replacing a model preserves a theorem only when the theorem quantifies over arbitrary predictions or the replacement satisfies every assumption. Safety mechanisms can act only while their required signal is available and the affected action remains reversible.

\subsection{Modularized Learning}
\label{subsec:modular}

Modularity is an engineering aid, not a prerequisite for a guarantee. Separating prediction from decision exposes a typed object $\hat y$; separating model from policy makes replacement obligations explicit; separating update from execution permits staged validation and rollback. Mathematical and service boundaries need not coincide, so the design record should state the objective, benchmark, state, and assumptions attached to each analyzed interface.

\subsection{Inference Graphs, and How Guarantees Compose}
\label{subsec:compose}

Model a pipeline as a directed acyclic graph whose nodes are learning-augmented components. Component $i$ may have consistency $c_i$, robustness $r_i$, and smoothness $f_i$ relative to its own instance and benchmark $\mathrm{OPT}_i$. Those triples are not an algebra for a pipeline because an output may change the next instance and benchmark. The following propositions are expository sufficient conditions derived to clarify this gap; they synthesize elementary inequalities rather than claim general composition theorems from the cited literature.

\begin{definition}[Inference graph]
\label{def:infgraph}
An \emph{inference graph} is a directed acyclic graph (DAG) $G = (V, E)$ in which each node $i \in V$ is a learning-augmented component with a declared triple $(c_i, r_i, f_i)$, and each edge $(i,j) \in E$ indicates that the output of $i$ is an input to the predictor or the decision rule of $j$. A node with no incoming edges consumes an external prediction. A node with no outgoing edges actuates.
\end{definition}

The first pattern is \emph{additive}: component costs sum and their benchmarks admit a common comparison.

\begin{proposition}[Expository condition: additive composition]
\label{prop:additive}
Fix the external instance and suppose $\mathrm{cost}_{\mathrm{sys}}=\sum_i\mathrm{cost}_i$ pathwise and $\mathrm{OPT}_{\mathrm{sys}}>0$. Let $\mathcal H_i$ determine reachable $I_i$, its prediction, and its error before component $i$ draws new randomness. Suppose every zero-additive-term component bound holds conditionally on each reachable history,
\[
  \mathbb E[\mathrm{cost}_i\mid\mathcal H_i]
  \leq a_i(\mathcal H_i)\,\mathrm{OPT}_i(I_i).
\]
This requires fresh randomness independent of prior history or a guarantee valid against adaptive instance selection. If $\mathrm{OPT}_{\mathrm{sys}}\geq\beta\sum_i\mathrm{OPT}_i(I_i)$ holds pathwise for $\beta\in(0,1]$, then
\begin{equation}
  \frac{\mathbb E[\mathrm{cost}_{\mathrm{sys}}]}{\mathrm{OPT}_{\mathrm{sys}}}
  \leq \frac{1}{\beta}\,
  \mathbb E\!\left[\max_i a_i(\mathcal H_i)\right].
    \label{eq_additive}
\end{equation}
Consequently, exact predictions give $c_{\mathrm{sys}}\leq\max_i c_i/\beta$, arbitrary predictions give $r_{\mathrm{sys}}\leq\max_i r_i/\beta$, and deterministic error caps $\eta_i\leq\bar\eta_i$ give $f_{\mathrm{sys}}(\bar{\boldsymbol\eta})\leq\max_i f_i(\bar\eta_i)/\beta$ for non-decreasing $f_i$.
\end{proposition}

\begin{proof}[Proof sketch]
By the tower property,
\[
\begin{aligned}
\mathbb E[\mathrm{cost}_{\mathrm{sys}}]
  &=\sum_i\mathbb E\!\left[\mathbb E[\mathrm{cost}_i\mid\mathcal H_i]\right] \\
  &\leq\mathbb E\!\left[\sum_i a_i(\mathcal H_i)\mathrm{OPT}_i(I_i)\right] \\
  &\leq\frac{\mathrm{OPT}_{\mathrm{sys}}}{\beta}
       \mathbb E\!\left[\max_i a_i(\mathcal H_i)\right].
\end{aligned}
\]
\end{proof}

The conditional guarantees and pathwise benchmark decomposition must be proved for the application; marginal component bounds are insufficient. Additive terms $b_i$ contribute $\sum_i\mathbb E[b_i]$, so multiplicative triples no longer suffice.

The second pattern is \emph{cascaded}: the output of component $i$ helps define the instance $I_{i+1}$ seen by component $i+1$. In this case a local competitive ratio compares the wrong pair of quantities to support a product rule.

\begin{proposition}[Expository condition: cascaded composition]
\label{prop:cascade}
Let $1\to2\to\cdots\to m$ be a chain in which component $i$ induces reachable $I_{i+1}$. Suppose system cost is the last-stage cost and $\mathrm{OPT}_1(I_1)>0$. Let $\mathcal H_m$ determine $I_m$ before the last component draws new randomness. Require $\mathrm{OPT}_m(I_m)>0$ on every reachable history, absent an additive zero-optimum convention, and suppose
\[
  \mathbb E[\mathrm{cost}_m(I_m)\mid\mathcal H_m]
  \leq r_m\,\mathrm{OPT}_m(I_m).
\]
This requires fresh randomness independent of $\mathcal H_m$ or a guarantee valid against adaptive input selection. For each edge, suppose finite $d_i\geq0$ gives the pathwise inequality
\[
    \mathrm{OPT}_{i+1}(I_{i+1}) \leq d_i\,\mathrm{OPT}_i(I_i)
\]
for every reachable transition. Then
\begin{equation}
    \mathbb E[\mathrm{cost}_{\mathrm{sys}}]
    \;\leq\;
    r_m\!\left(\prod_{i=1}^{m-1} d_i\right)\mathrm{OPT}_1(I_1).
    \label{eq_cascade}
\end{equation}
\end{proposition}

\begin{proof}[Proof sketch]
The tower property bounds $\mathbb E[\mathrm{cost}_m]$ by $r_m\mathbb E[\mathrm{OPT}_m(I_m)]$. Iterating the pathwise edge inequalities bounds the latter benchmark by $(\prod_{i=1}^{m-1}d_i)\mathrm{OPT}_1(I_1)$, proving~\eqref{eq_cascade}.
\end{proof}

The factors $d_i$ do not follow from $(c_i,r_i,f_i)$. Earlier-stage ratios matter only if they control downstream benchmark distortion; perfect component consistency therefore need not imply perfect system consistency.

Smoothness composes only under an additional regularity condition, and stating it clarifies what an inference graph must declare about its edges.

\begin{proposition}[Expository condition: error propagation]
\label{prop:errprop}
For a path $1\to2\to\cdots\to m$, equip each component-output space and downstream-input space with metrics $d_i^{\mathrm{out}}$ and $d_{i+1}^{\mathrm{in}}$. Let the ideal edge map $F_i$ be $L_i$-Lipschitz, let $o_i^*$ and $o_i$ be the ideal and actual outputs with $d_i^{\mathrm{out}}(o_i,o_i^*)\leq g_i(\eta_i)$, and let $x_{i+1}^*=F_i(o_i^*)$. If the actual downstream input $x_{i+1}$ also satisfies $d_{i+1}^{\mathrm{in}}(x_{i+1},F_i(o_i))\leq \eta_{i+1}^{\mathrm{own}}$, then the triangle inequality gives
\begin{equation}
    \eta_{i+1}:=d_{i+1}^{\mathrm{in}}(x_{i+1},x_{i+1}^*)
    \;\leq\; L_i\, g_i(\eta_i) + \eta_{i+1}^{\mathrm{own}}.
    \label{eq_errprop}
\end{equation}
If these conditions hold on every edge of the path, the recursion can be unrolled and a last-stage smoothness bound evaluated at the resulting upper bound on $\eta_m$. A general DAG with multiple parents additionally needs a joint map from the product of parent-output spaces and a compatible product metric; applying the single-parent inequality independently does not establish such a bound. If an edge lacks an appropriate metric, Lipschitz bound, or additive perturbation model, this argument supplies no transferred smoothness guarantee.
\end{proposition}

Component guarantees must therefore be accompanied by interface facts: which instance is produced, how benchmarks relate, which costs have already been incurred, and which sensitivity property supports error propagation. Safety placement has no universal dominance rule: a monitor can constrain only observable costs and reversible actions, and a fallback requires its own proof. Pipeline depth matters through these edge effects, not through automatic multiplication of local ratios.

\subsection{Control Loops Separated by Timescale}
\label{subsec:loops}

Feedback operates on application-specific timescales determined by observability and reversibility. A fast mechanism can use only an immediately available hard constraint or accountable cost; prediction may be restricted to actions that preserve feasibility~\cite{agrawal2023safety}. Delayed labels can support recalibration or retraining but cannot repair an irreversible action. Drift diagnostics, intervention rates, and confidence are proxies rather than certificates. Portfolio selection has a guarantee only when its losses and comparator satisfy a problem-specific reduction~\cite{dinitz2022portfolios}.

\subsection{Formal Robustifiers and Operational Safety Mechanisms}
\label{subsec:safety}

Formal robustifiers, feasibility checks, circuit breakers, restricted action spaces, and fallbacks must not be treated as interchangeable. A formal robustifier has a problem-specific rule and proved guarantee; an operational mechanism does not inherit a competitive ratio by monitoring cost or running a baseline. Evaluation should state the observed signal, maintained state, reversible action, fallback feasibility, and evidence justifying a transition. Intervention rate is useful operationally but should be reported with objective value, violations, inference and monitoring overhead, and fallback frequency.

\subsection{Case Studies}
\label{subsec:cases}

The architecture above provides a common accounting lens for selected deployed systems; it is not a claim that they implement one shared design. Table~\ref{tab:cases} separates eight systems or system classes that were previously easy to conflate. Each row states what is predicted, what evidence exists, what prediction costs are visible, and what happens when the predictor is wrong.

\begin{table}[t!]
\centering
\setlength{\tabcolsep}{3pt}
\caption{Eight systems-oriented examples viewed through the LAS questions. Each row refers only to the named system or class; formal guarantees, operational safeguards, and empirical behavior are not transferred across rows.}
\label{tab:cases}
\scriptsize
\begin{tabular}{@{}
    >{\raggedright\arraybackslash}p{1.55cm}
    >{\raggedright\arraybackslash}p{2.62cm}
    >{\raggedright\arraybackslash}p{3.22cm}
    >{\raggedright\arraybackslash}p{2.38cm}
    >{\raggedright\arraybackslash}p{3.24cm}
@{}}
\toprule
\textbf{System} & \textbf{Predicted quantity} & \textbf{Evidence and guarantee} & \textbf{Prediction cost} & \textbf{Behavior under error} \\
\midrule
LRB~\cite{song2020lrb} & Reuse distance or next access for eviction & Empirical cache performance; no native competitive guarantee & Inference on the eviction path & Error can increase misses; no theorem-level worst-case cap follows from LRB alone \\
\addlinespace[2pt]
Parrot~\cite{liu2020parrot} & Eviction priority learned by imitation & Empirical cache performance; no native competitive guarantee & Policy inference on eviction & Error can change eviction order; robustness remains empirical for Parrot itself \\
\addlinespace[2pt]
PGM~\cite{ferragina2020pgm} & Key position with certified segment error & Static recursive index has $O(\!\log m+\log(\epsilon+1))$ query time and $\Theta(m)$ space & Model evaluation and bounded local search & Correctness is retained by searching the certified interval; update costs use a separate dynamic analysis \\
\addlinespace[2pt]
ALEX~\cite{ding2020alex} & Key position and layout cost & Correct index semantics plus empirical latency and memory results; no $(c,r,f)$ contract & Model evaluation, inserts, splits, and expansion & Adaptive structure handles updates empirically; performance varies by workload~\cite{wongkham2022updatable} \\
\addlinespace[2pt]
Decima~\cite{mao2019decima} & Scheduling policy from job DAG and cluster state & Empirical system objective; no competitive guarantee & Offline trace training and online policy inference & Behavior beyond the evaluated workloads is an empirical question; no LAA theorem bounds it \\
\addlinespace[2pt]
Autopilot~\cite{rzadca2020autopilot} & Resource settings from workload observations & Production evaluation and operational constraints; no $(c,r,f)$ theorem & Continuous observation and recommendation & Limits and gradual changes are safeguards, not competitive-ratio certificates \\
\addlinespace[2pt]
Bao~\cite{marcus2021bao} & Performance of optimizer hint sets & Empirical query-performance evidence; action space is restricted to hint sets & Plan scoring and periodic adaptation & Restriction narrows possible plans but does not itself bound regret or competitive ratio \\
\addlinespace[2pt]
Cardinality estimators~\cite{wang2021cardinality,sun2022cardinality} & Relation or subplan cardinalities & Accuracy and downstream latency are empirical and workload dependent & Training and inference per selected estimator & Optimizer behavior mediates error; estimator accuracy alone does not determine plan cost \\
\bottomrule
\end{tabular}
\end{table}

The evidence falls into three declared categories: formal theorem, operational safeguard, and empirical evaluation. A later formal robustifier must be treated separately from its empirical base. In particular, the \textsc{Guard}$\&$LRB and \textsc{Guard}$\&$Parrot guarantees follow from the paper's proofs that these base algorithms are RB-following in its model, together with the robustification theorem~\cite{chen2025robustifying}. This conclusion remains specific to the stated paging and prediction assumptions. Learned storage routing likewise uses predicted latency as a hint but requires separate assumptions about routing constraints and fallback behavior~\cite{hao2020linnos}.

\subsection{Deployment Practice}
\label{subsec:practice}

Deployment practice follows the same accounting. Report interventions together with objective values, violations, inference cost, and policy version. Shadow execution can test latency and proposed actions, but policy-dependent state or labels require explicit replay, off-policy, or canary assumptions. Rollback must handle state conversion and already-incurred costs. Incident reviews should attribute failures to signal, prediction, decision, or actuation. Finally, every prediction-consuming edge should record its object, benchmark relation, sensitivity evidence, and remaining reversible costs; depth alone is not a competitive-ratio parameter.

\section{Semantic Predictions and Large Language Models}
\label{sec:llm_integration}

The learning-augmented literature usually assumes a prediction object with task-defined semantics: a number, ranking, distribution, or combinatorial structure. A foundation model may instead emit free-form text. This section asks what changes when text must be converted into the prediction space of Figure~\ref{fig:las_arch}. It is deliberately conservative: operational uses are growing, whereas results of the form in Definitions~\ref{def:consistency}--\ref{def:smoothness} with a language model in the prediction layer remain sparse.

\subsection{What Semantic Prediction Adds, and What It Does Not}
\label{subsec:semantic}

Foundation models can map unstructured context, e.g., tickets, changelogs, logs, or operator notes, to a declared task object. This is not unique to them, so accuracy and cost should be compared with parsers, information-extraction systems, and other NLP models.

Token-sequence likelihood is not automatically a calibrated probability of task correctness; sequence-likelihood calibration and semantic-uncertainty methods address narrower versions of this mismatch~\cite{zhao2023slic,farquhar2024semantic}. Free-form text therefore cannot be inserted into P1--P5 without an interface that maps it to the prediction space and defines an error measure. Once such an interface is supplied, any compatible mechanism may apply, but its theorem must account for the adapter's failure modes rather than treating semantics as a new guarantee.

Nagawanshi et al.~\cite{nagawanshi2026moment} provide a boundary case: time-decay streaming theorems assume a suffix-compatible heavy-hitter oracle, while experiments instantiate it with CountSketch, an LSTM, ChatGPT, and Gemini. The experiments do not prove that a named model satisfies the oracle condition on arbitrary streams. This is conditional-interface evidence, not an unconditional LLM guarantee.

\subsection{Four Integration Patterns}
\label{subsec:llm_patterns}

Four roles are useful to distinguish. Their risk is not totally ordered: it depends on validation, runtime authority, and whether a human acts on the output.

\emph{As a predictor.} The model emits $\hat y$, e.g., a duration or job order. P1--P5 apply only after parsing and establishing the theorem's error assumptions; calibration is not universally sufficient.

\emph{As an adapter.} The model extracts a structured input from text. A guarantee survives only if it tolerates arbitrary outputs or validation enforces its assumptions; range restriction alone gives no Lipschitz bound.

\emph{As a policy synthesizer.} The model proposes an offline rule or fallback. The artifact still needs the verification and deployment controls applied to human-written code.

\emph{As an explainer.} The model renders telemetry for an operator. Even outside the automated path, it can affect consequential actions, so factuality, provenance, and uncertainty remain necessary.

The architectural question is where an unvalidated output can change state. The interface should expose only the needed authority, with validation proportional to consequences.

\subsection{The Calibration Problem}
\label{subsec:llm_calib}

Suppose a model is used as a predictor and the surrounding algorithm needs confidence to set $\lambda$. Token-level likelihood need not rank generated sequences by task quality~\cite{zhao2023slic}, and its relationship to correctness can vary with prompts, decoding, and model versions. Agreement across sampled completions measures self-consistency rather than labeled accuracy; semantic entropy is one empirically studied alternative uncertainty signal, not a universal correctness certificate~\cite{farquhar2024semantic}.

One possible route is a conformal or other distribution-free wrapper that converts task outputs and a calibration sample into a prediction set. Standard split-conformal coverage uses exchangeability; extensions to non-exchangeable data incur model-specific coverage penalties~\cite{oliveira2024split}. Marginal coverage is not automatically a bound on downstream cost. Existing online algorithms with UQ make the uncertainty statement part of a problem-specific input model~\cite{sun2024uncertainty}; they do not provide a generic reduction from conformal coverage or language-model confidence to a P5 cost guarantee. Such a reduction would require a separate end-to-end analysis.

When calibrated confidence is unavailable, available responses are application dependent, e.g., abstaining, requesting review, constraining the action, or using a fallback with an independently justified guarantee. A generic safety mechanism does not automatically bound the resulting cost or harm (\S\ref{subsec:safety}).

\subsection{Failure Modes Specific to Generative Components}
\label{subsec:llm_failures}

Four failure mechanisms merit explicit treatment. Some have numerical analogues, while text and tool use add interfaces; P1--P5 do not remove them.

\emph{Fabrication.} A language model can emit a fluent, specific, and wrong statement. Numerical predictors can also be confidently wrong, so any switching rule that consumes confidence must assume or verify a relationship between that confidence and task error.

\emph{Injection.} The signal layer may carry attacker-influenced text such as filenames or log lines echoing request content. Indirect prompt-injection demonstrations show that retrieved data can redirect an integrated model or tool workflow~\cite{greshake2023prompt}. Numerical systems also face adversarial inputs, but instruction/data ambiguity creates a distinct interface risk. Treat such text as untrusted data, isolate tool authority, and validate structured outputs.

\emph{Non-determinism and version drift.} Sampling can produce different outputs across calls, and model updates can change behavior. A guarantee must quantify over that randomness and specify the version or model class to which its assumptions apply. Pinning, regression tests, and staged promotion help operationally but do not replace that analysis.

\emph{Feedback contamination.} If model outputs are logged and later ingested as training data or context, model-generated errors can re-enter later inputs. Experiments on recursively generated training data demonstrate one concrete degradation mechanism, although they do not characterize every operational feedback loop~\cite{shumailov2024collapse}. This is one form of the endogeneity described in S2. Provenance labels can support separate treatment of model-generated content in later training and evaluation.

\subsection{The State of the Evidence}
\label{subsec:llm_evidence}

We separate three claims that are often stated together.

What is \emph{established}: an AIOps survey documents language-model applications to log analysis, incident triage, and root-cause analysis~\cite{zhang2026aiops}, but does not itself establish production prevalence or an LAA guarantee. Time-decay moment estimation separately proves results for a declared heavy-hitter oracle and tests language-model implementations~\cite{nagawanshi2026moment}.

What remains \emph{conditional}: a language model can be a predictor if an adapter defines the object and its output satisfies the theorem's premise. No surveyed work proves this unconditionally for a particular model over the theorem's full instance class; calibration alone is insufficient.

What is \emph{unsupported without a new model and proof}: that semantic reasoning improves an existing guarantee automatically. Richer side information can change the formal information structure and can therefore change achievable bounds, but only relative to a clearly specified predictor class, error model, and benchmark. Lower bounds proved for a fixed information model cannot simply be transferred to a richer one, nor can they be declared overcome without re-analysis.

Thus, a conditional oracle reduction exists for one streaming family~\cite{nagawanshi2026moment}, but no surveyed result unconditionally verifies its premise for a particular language model or transfers a general LAA guarantee to free-form output.

\section{Open Challenges}
\label{sec:open_challenges}

Four cross-cutting problems concern predictor cost, composition, endogenous error, and empirical comparability.

\subsection{Pricing the Prediction}
\label{subsec:overhead}

The competitive ratios of Section~\ref{sec:domains} charge nothing for obtaining $\hat y$. One survey-derived accounting convention is the following definition, which is a modeling proposal rather than a theorem.

\begin{definition}[$\kappa$-augmented competitive ratio]
\label{def:kappa_cr}
Let $\kappa \geq 0$ be a declared conversion from one predictor invocation to the units of the objective, and let $q_{\mathcal A}(\sigma,\hat y)$ be the possibly random number of invocations made by algorithm $\mathcal{A}$. Define the augmented expected cost
$\mathrm{cost}_\kappa(\mathcal{A}(\sigma,\hat y)) = \mathbb E[\mathrm{cost}(\mathcal{A}(\sigma,\hat y)) + \kappa\, q_{\mathcal A}(\sigma,\hat y)]$, where the expectation is over the algorithm's internal randomness, and
\begin{equation}
    \mathrm{CR}_\kappa(\mathcal{A}) \;=\; \sup_{\sigma,\,\hat y} \frac{\mathrm{cost}_\kappa(\mathcal{A}(\sigma, \hat y))}{\mathrm{OPT}(\sigma)},
    \label{eq_kappa_cr}
\end{equation}
where the supremum ranges over pairs with $\mathrm{OPT}(\sigma)>0$, and $\mathrm{OPT}$ remains the prediction-free offline optimum and is not charged $\kappa$.
\end{definition}

For a fixed pair on which $q_{\mathcal A}=T$ almost surely, the augmented ratio gains $\kappa T/\mathrm{OPT}(\sigma)$; one may not instead add separate suprema. Exact multiplicative $1$-consistency against the uncharged optimum is therefore impossible whenever an exact-prediction instance has $\kappa\,\mathbb E[q_{\mathcal A}]>0$, absent an additive term or another benchmark. Drygala et al.~\cite{drygala2023costly} analyze cost and query timing for specific stopping problems; reduced-query models instead budget calls~\cite{im2022parsimonious,sadek2024reduced}. Neither yields a universal three-way frontier.

\subsection{A Calculus for Composition}
\label{subsec:compose_open}

Propositions~\ref{prop:additive}-\ref{prop:errprop} require benchmark decomposition, distortion factors, or metric sensitivity; node triples alone do not determine a system triple. A fuller calculus must model correlated claims, policy-dependent cycles, objectives, state constraints, and limited fallback capacity.

\subsection{Endogenous, Adversarial and Strategic Error}
\label{subsec:adversarial}

Bounds conditional on realized $\eta$ do not model how decisions change later observations. Performative prediction studies decision-dependent distributions~\cite{perdomo2020performative}, feedback control studies explicitly modeled dynamics~\cite{astrom2008feedback}, and strategic classification studies utility-driven adaptation~\cite{hardt2016strategic}. These distinct tools do not by themselves yield a consistency--robustness guarantee under policy-dependent data.

\subsection{Benchmarking and Reproducibility}
\label{subsec:bench}

A $(c,r)$ pair follows from definitions and proof; experiments test implementations and workloads. Comparable benchmarks should declare traces, shifts, error protocols, objectives, feasibility, interventions, predictor calls, latency, prediction type, trust, splits, and baselines~\cite{chledowski2021robust,wongkham2022updatable,wang2021cardinality}. They complement rather than validate universal guarantees.

\section{Conclusion}
\label{sec:conclusion}

Learning-augmented algorithms give theorem-specific consistency, robustness, and error-dependent guarantees. Our two axes record five non-exclusive construction mechanisms and same-model matching-bound status; neither is a maturity ranking. This separates proved frontiers, achieved upper bounds, and empirical systems evidence.

Systems also incur prediction cost, policy-dependent feedback, and component interactions. Definition~\ref{def:kappa_cr} prices calls in one model, while Propositions~\ref{prop:additive}--\ref{prop:errprop} require explicit objective, benchmark, state, and error relations for composition. Operational safeguards become guarantees only through problem-specific analysis.

\IfFileExists{ACM-Reference-Format.bst}%
  {\bibliographystyle{ACM-Reference-Format}}%
  {\ClassWarningNoLine{acmart}{ACM-Reference-Format.bst not found; falling back to plainnat}%
   \bibliographystyle{plainnat}}
\bibliography{arxiv}

\end{document}